\documentclass{article}

\usepackage{arxiv}

\usepackage[utf8]{inputenc} % allow utf-8 input
\usepackage[T1]{fontenc}    % use 8-bit T1 fonts
\usepackage{hyperref}       % hyperlinks
\usepackage{url}            % simple URL typesetting
\usepackage{booktabs}       % professional-quality tables
\usepackage{amsfonts}       % blackboard math symbols
\usepackage{nicefrac}       % compact symbols for 1/2, etc.
\usepackage{microtype}      % microtypography
\usepackage{lipsum}		% Can be removed after putting your text content
\usepackage{graphicx}
\usepackage{doi}
\usepackage{amsthm}
\usepackage{amsmath}
\theoremstyle{plain}% Theorem-like structures provided by amsthm.sty
\newtheorem{theorem}{Theorem}
\newtheorem{lemma}[theorem]{Lemma}

\newtheorem{proposition}[theorem]{Proposition}

\theoremstyle{definition}
\newtheorem{definition}{Definition}

\theoremstyle{remark}
\newtheorem{remark}{Remark}

\title{A group-theoretic framework for machine learning in hyperbolic spaces}

\author{Vladimir Ja\' cimovi\' c \\
	Faculty of Natural Sciences and Mathematics\\
	University of Montenegro\\
	Cetinjski put bb., 81000 Podgorica\\
	Montenegro\\
	\texttt{vladimirj@ucg.ac.me} \\
}

\renewcommand{\shorttitle}{A group-theoretic framework for machine learning in hyperbolic spaces}

\hypersetup{
pdftitle={A group-theoretic framework for machine learning in hyperbolic spaces},
pdfsubject={-},
pdfauthor={Vladimir Ja\' cimovi\' c},
pdfkeywords={Barycenter, statistical model, hyperbolic gradient, maximum likelihood, conformal invariance},
}

\begin{document}
\maketitle

\begin{abstract}
Embedding the data in hyperbolic spaces can preserve complex relationships in very few dimensions, thus enabling compact models and improving efficiency of machine learning (ML) algorithms. The underlying idea is that hyperbolic representations can prevent the loss of important structural information for certain ubiquitous types of data. However, further advances in hyperbolic ML require more principled mathematical approaches and adequate geometric methods. The present study aims at enhancing mathematical foundations of hyperbolic ML by combining group-theoretic and conformal-geometric arguments with optimization and statistical techniques.
Precisely, we introduce the notion of the mean (barycenter) and the novel family of probability distributions on hyperbolic balls. We further propose efficient optimization algorithms for computation of the barycenter and for maximum likelihood estimation. One can build upon basic concepts presented here in order to design more demanding algorithms and implement hyperbolic deep learning pipelines.
\end{abstract}

% keywords can be removed

\section{Introduction}\label{sec:1}
The idea of learning representations in hyperbolic spaces has rapidly gained prominence in the last decade, attracting a lot of attention and motivating extensive investigations. This rise of interest was partly launched by statistical-physical studies \cite{KPKVB} which have shown that distinctive properties of complex networks are naturally preserved in negatively curved continuous spaces. Since complex networks are ubiquitous in modern science and everyday life, this relation with hyperbolic geometry provided a valuable hint for low-dimensional representations of hierarchical data \cite{MTCBC}.

More generally, structural information of any hierarchical data set may be better represented in negatively curved manifolds rather than in flat ones. This further implies that hyperbolic geometry provides a suitable framework for simultaneous learning of hypernymies, similarities and analogies. This hypothesis triggered the interest of many data scientists and machine learning (ML) researchers in hyperbolic geometry. Nowadays, hyperbolic ML is a rapidly developing young subdiscipline within the broader field of {\it geometric deep learning} \cite{BBCV}.

Recent advances in hyperbolic representations of data open exciting perspectives and horizons for the near future. Encouraging results have been reported in natural language processing (NLP) \cite{TBG,NK2,LW}, computer vision \cite{MAK-RGY,KMUOL}, molecular interactions \cite{Poleksic,CYPVLL}, recommender systems \cite{CHWDDV,LYHCXW}, knowledge graphs \cite{CWJ} and brain research \cite{Baker,AS,JFB}. In addition, some researchers devised a general hyperbolic-geometric approach to reinforcement learning (RL), due to the hierarchical nature of Markov decision processes \cite{CCBH}.

Despite encouraging results, benefits and the scope of hyperbolic ML are still questionable. First, the data never explicitly appear as points in hyperbolic spaces (with the possible exception of high-energy physics, due to the hyperbolic geometry of spacetime \cite{BAORMK}). This fact imposes the necessity of demanding embedding algorithms as a precondition for hyperbolic ML. Second, traditional ML architectures strongly rely on inherently Euclidean techniques (such as addition of vectors, multiplying vectors by scalars, scalar product and averaging) which are not available in hyperbolic spaces. Therefore, hyperbolic ML requires modification of traditional architectures or development of new models \cite{Gulcehre,GBH,TNN,CW2}. Third, it is still unclear for which kinds of data and problems hyperbolic representations provide an overall advantage over the Euclidean ones.\footnote{Notice that the situation with spherical ML is considerably simpler, as spherical geometry of certain data sets is apparent. For instance, when learning orientations or rotations in the space, spherical representations are obvious choice. Another example of spherical data are geospatial earth observations.} While intuitive and theoretical arguments in favor or against hyperbolic ML may be discussed, the definitive disposition about its significance is to be provided by experimental results.

On the other hand, experimental results strongly depend on specific implementations and the underlying mathematical techniques. Learning low-dimensional representations is one of central issues in modern ML. This issue underlies the idea of latent space. ML architectures with latent spaces provide the framework for validation of the idea of hyperbolic ML. Indeed, training models with latent spaces boils down to learning low-dimensional data representations using optimization algorithms, such as Bayesian or variational inference. Endowing the latent space with the negative curvature seems like a potentially advantageous approach whenever dealing with hierarchical data.\footnote{Geometry of the latent space is one of the key conceptual problems in ML \cite{AHH,WHPD}. We refer to \cite{SAC} for a nice overview of this topic. Experiments with normalizing flows in hyperbolic latent spaces have been reported in \cite{BSLPH}.} In particular, one of the most significant benefits expected from hyperbolic representations is great reduction of the dimensionalty, allowing for drastically more compact models. Hyperbolic latent spaces with a very low dimension (up to ten) might be sufficient in very demanding tasks with high-dimensional data \cite{NK1}. This idea motivated experiments with variational auto-encoders \cite{MLMTT}, bio-encoders \cite{CYPVLL}, collaborative filtering \cite{Poleksic,LYHCXW} and other architectures. Strictly speaking, architectures with variational (probabilistic) latent spaces (such as variational auto-encoders or probabilistic matrix factorization algorithms) impose non-Euclidean setup in any case. Indeed, such models map the data into probability distributions and spaces of probability distributions are naturally endowed with the Fisher information metric which imposes non-zero curvature.\footnote{For instance, standard choice for the latent space is the family of Gaussian distributions with diagonal covariance matrix. However, the family of univariate Gaussian distributions endowed with the Fisher information metric is isomorphic to the hyperbolic disc \cite{CSS}. Hence, from the information-geometric point of view, latent spaces consisting of Gaussians with diagonal covariance matrix are hyperbolic multidiscs.} This implies that training any model with variational latent space is an optimization problem over a Riemannian manifold. Therefore, training algorithms should perform an update along the natural gradient rather than along Euclidean ("vanilla") gradient.

The above discussion brings us to the central question: {\it Does hyperbolic ML leverage an appropriate mathematical framework in its practices and implementations?} Indeed, it is possible that the full potential of the idea remains uncovered because mathematical notions and techniques are not completely adequate. So far, most experiments in hyperbolic ML exploited gyrovector spaces \cite{Ungar}, exponential map and projected stochastic gradient descent. In addition, statistical modeling was based on the family of "wrapped normal distribution on hyperbolic disc" \cite{Poleksic,NYFK,CLPK}. We refer to \cite{PVMSZ} for an overview of mathematical techniques used in hyperbolic ML. On the other hand, complex analysis and conformal geometry are very infrequently used in ML despite the fact that they provide the natural framework for investigations of hyperbolic spaces.\footnote{We point out some recent studies \cite{AK-RM,CGNR} employing the Buseman function for representation learning, thus bringing some conformal-geometric concepts into ML.} For instance, hyperbolic balls are homogeneous spaces with symmetry groups which are isomorphic to the Lorentz groups. Optimizing in hyperbolic balls boils down to learning actions of these groups. In general, learning actions/generators of Lie groups is a parallel important direction in Geometric DL \cite{CW1,Barbaresco}.

The main goal of the present paper is to enhance theoretical foundations of hyperbolic ML, by combining group-theoretic and geometric tools with optimization and statistics.

%The main goal of this paper is to propose alternative mathematical foundations for hyperbolic ML. We suggest to abandon operations and concepts from Euclidean geometry that are not natural in hyperbolic spaces. Instead, the proposed framework leverages the fact that hyperbolic balls are homogeneous spaces with the great groups of isometries acting on them. Therefore, in many cases learning over hyperbolic balls can be realized through learning actions of these groups. These are Lie groups, so learning their infinitesimal generators. The idea of conformal invariance, a sophisticated and pretty well understood mathematical theory with many new insights in previous decades.

%Hence, our starting point is that some basic operations, such as addition are not available. Moreover, the Gaussian probability distributions are not feasible for encoding uncertainties. Instead, symmetries and group actions.

For the further exposition it is important to underline that any discussion on data embeddings and ML algorithms implies the notion of distance. In general, one can embed the data and perform ML in any metric space $(X,d)$. Nevertheless, in order to design meaningful and efficient models, certain prerequisites should be ensured in the first place. We highlight four basic questions which should be answered before any ML experiments in $(X,d)$.

$i)$ Given the set of points $x_1,\dots,x_N$ in $X$, what is their mean?

The mean should be defined in such a way to satisfy certain common sense properties. Most important, if $a \in X$ is the mean of points $x_1,\dots,x_N$ and $g : X \to X$ is an isometry, then $g(a)$ must be the mean of points $g(x_1),\dots,g(x_N)$.

$ii)$ How can we compute the mean?

The algorithm for computing the mean should be reasonably efficient.

$iii)$ How to sample a random point in $X$?

Parametric statistical models assume a certain family of probability distributions on $X$. This family should satisfy certain properties. First, the family should be invariant with respect to actions of the symmetry group on $X$. In other words, if we act on a probability measure by an isometry, we should obtain the probability measure which belongs to the family. Second, probability distributions should have well defined and easily computable mean (mathematical expectation). Finally, if we act on a probability measure by an isometry, then the mean should be transformed by the same isometry.

$iv)$ Given the family of probability distributions on $X$ and samples $x_1,\dots,x_N$ from a probability distribution belonging to this family, how can we estimate the parameters?

The maximum likelihood estimation procedure should be reasonably efficient.

\begin{remark}
Underline that in Euclidean spaces questions $i)-iv)$ have decisive answers, with the corresponding mathematical tools universally accepted and implemented. The mean (barycenter) of $N$ vectors $x_1,\dots,x_N$ is their average $\frac{1}{N} \sum x_i$. The universal statistical model in $n$-dimensional Euclidean space the Gaussian family ${\cal N}(a,\Sigma)$, where $a \in \mathbb{R}^n$ is mathematical expectation and $\Sigma$ is $n \times n$ (positive definite) covariance matrix. The maximum likelihood estimation of parameters $a$ and $\Sigma$ yields an optimization problem that allows for relatively efficient numerical solution. Optimization in Euclidean spaces is based on the gradient descent method with further improvements and modifications (stochastic gradient descent, momentum, etc.). There is also an alternative approach based on evolutionary algorithms for the black-box optimization, most notably the CMA ES algorithm which also employs the Gaussian family \cite{HO}.
\end{remark}

In the present paper we provide detailed and rigorous answers to the questions $i)-iv)$ for the case when $(X,d)$ is a hyperbolic ball. The exposition is organized along the following lines.
The next Section contains mathematical preliminaries about hyperbolic balls. In sections \ref{sec:3} and \ref{Moeb_Poin_disc} we consider the minimal and most popular model of hyperbolic geometry: the two-dimensional Riemannian manifold known as Poincar\' e disc. In Section \ref{sec:3} we define the barycenter in the Poincar\' e disc and assert the conformal invariance property. Furthermore, we introduce so-called Poincar\' e swarms as a powerful computational model and demonstrate that these swarms implement the gradient descent method in hyperbolic metric for computing the barycenter. In Section \ref{Moeb_Poin_disc} we introduce a conformally invariant family of probability distributions in the Poincar\' e disc and propose the optimization algorithm for maximum likelihood estimation. In sections \ref{Bary_balls} and \ref{sec:6} the analysis is extended to higher-dimensional balls, including two non-equivalent models in even dimensions. Section \ref{sec:7} contains concluding remarks and an outlook on the future advances and applications.

\section{Preliminaries on hyperbolic balls}\label{sec:2}

Minimal models of hyperbolic geometry are two-dimensional manifolds with constant negative curvature. The most popular choice in hyperbolic ML is the Poincar\' e disc (also named conformal disc). This manifold can be elegantly defined and studied using the complex-analytic tools, because the group of isometries of the Poincar\' e disc coincides with the (sub)group of disc-preserving M\" obius (linear-fractional) transformations in the complex plane.

More generally, the unit ball in $d$-dimensional real vector space can be endowed by the metric imposing the constant negative curvature on it. This is the way of introducing Poincar\' e balls which are $d$-dimensional hyperbolic manifolds for any integer $d \geq 2$. On the other hand, one can introduce unit balls as subsets of the complex vector space $\mathbb{C}^m$. In this case, there exists an alternative metric which also imposes the constant negative curvature. Hence, there are two non-equivalent models of even-dimensional hyperbolic balls, named Poincar\' e and Bergman balls. Underline that for $m=1$ (i.e. $d=2$) both models reduce to the Poincar\' e disc. In that respect, both Poincar\' e and Bergman balls are extensions of the Poincar\' e disc to higher dimensions. We refer to the book \cite{Stoll} for the rigorous exposition on hyperbolic balls and their symmetry groups and to \cite{Parker} for lectures on complex hyperbolic geometry, including Bergman balls.

\subsection{Poincar\' e disc}

Consider the open disc in the complex plane $\mathbb{B}^2 = \{ z \in \mathbb{C} \; |z|<1 \}$. If $v$ is the tangent vector to $\mathbb{B}^2$ at a point $z$, with the Euclidean norm $|v|_{Eucl}$, then the hyperbolic norm of $v$ is
$$
|v|_{hyp} = \frac{1}{1-|z|^2} |v|_{Eucl}.
$$
The open disc with the norm $|v|_{hyp}$ is the Riemannian manifold called {\it the Poincar\' e disc}.

Further, denote by $\mathbb{G}_2$ the group of M\" obius transformations acting on the complex plane of the following form
\begin{equation}
\label{Mobius}
g_a(z) = e^{i \theta} \frac{a-z}{1-\bar a z}, \quad \theta \in [0,2 \pi), \; a \in \mathbb{B}^2.
\end{equation}
It is easy to verify that transformations \eqref{Mobius} map the unit disc onto itself. Hence, $\mathbb{G}_2$ is subgroup of the group of all M\" obius transformations acting on the complex plane.

The (sub)group $\mathbb{G}_2$ is isomorphic to the Lie group $SU(1,1)$ of matrices of the form
$$
\left(
\begin{array}{cc}
a & b \\
- \bar b & \bar a
\end{array}
\right),
 \mbox{  where  } a,b \in \mathbb{C}, \quad |a|^2 + |b|^2 = 1.
$$
More precisely, $\mathbb{G}_2$ is isomorphic to the quotient group $PSU(1,1) = SU(1,1) / \pm I$.

Denote by $d \lambda(z)$ the Lebesgue measure in the complex plane. Then the hyperbolic measure in $\mathbb{B}^2$ reads
\begin{equation}
\label{hyp_meas_disc}
d \Lambda(z) = \frac{d \lambda(z)}{(1-|z|^2)^2}.
\end{equation}

Group ${\mathbb G}_2$ operates on the set ${\cal P}(\mathbb{B}^2)$ of all probability measures on $\mathbb{B}^2$. Given a measure $\mu \in {\cal P}(\mathbb{B}^2)$ we will use the notation $g_* \mu$ for the pullback measure defined as
\begin{equation}
\label{pull_back}
g_* \mu (A) = \mu(g^{-1}(A)), \mbox{ for any Borel set } A \subseteq \mathbb{B}^2.
\end{equation}

\subsection{Poincar\' e ball}

Denote by $\langle x,y \rangle = x_1y_1 + \cdots + x_d y_d$ the scalar product in $\mathbb{R}^d$ and $|x| = \sqrt{\langle x,x \rangle}$.

\begin{definition}
\label{Poin_ball_def}
Consider the set $\{x \in \mathbb{R}^d \; : \; |x|<1\}$ equipped with the metric $$g_x(u,v) = \frac{\left<u,v\right>}{(1-|x|^2)}, u, v\in \mathbb{R}^d.$$ This manifold is named the $d$-dimensional Poincar\' e ball and denoted $(\mathbb{B}^d, g)$.
\end{definition}

Hyperbolic measure in $\mathbb{B}^d$ is defined as
\begin{equation}
\label{hyp_meas_ball}
d \Lambda(x) = \frac{d \lambda(x)}{(1-|x|^2)^d}.
\end{equation}
where $d \lambda(x)$ denotes the Lebesgue measure in $\mathbb{R}^d$.

Let $x,y \in \mathbb{B}^d$. The Poincar\'e distance on $\mathbb{B}^d$ is given by
\begin{equation}\label{pome}d_h(x,y)=\frac{1}{2}\log \frac{1+R}{1-R},\end{equation} where
 \begin{equation}
 \label{rho}
 R=\frac{|x-y|}{\sqrt{\rho(x,y)}} \mbox{ and } \rho(x,a)=|x-a|^2+(1-|a|^2)(1-|x|^2).
 \end{equation}
Consider the set of M\"obius transformations of the unit ball given by the following formula
\begin{equation} \label{Mobius_ball}
h_a(x)= A \frac{a|x-a|^2+(1-|a|^2)(a-x)}{\rho(x,a)},
\end{equation}
where $A$ is the orthogonal transformation of the Euclidean space $\mathbb{R}^d$.

Transformations \eqref{Mobius_ball} map the unit ball in $\mathbb{R}^d$ onto itself. Furthermore, it is easy to see that $h_c^{-1}(x)=h_c(x)$ for every $c\in \mathbb{B}^d$. Transformations \eqref{Mobius_ball} with the operation of multiplication (composition) constitute a group. We denote this group by ${\mathbb G}_d$. Notice that $\mathbb{G}_d$ is isomorphic to the Lorentz group $SO^+(d,1)$. Transformations \eqref{Mobius_ball} are isometries of $\mathbb{B}^d$, that is - they preserve the distances \eqref{pome}.

We conclude this subsection with several formulae that may be useful in the sequel, see \cite{Ahlfors}

\begin{equation}\label{poMT}d_h(x,y)=\frac{1}{2} \log \frac{1+|h_a(x)|}{1-|h_a(x)|};\end{equation}

\begin{equation}\label{jaka}(1-|h_a(x)|^2)=\frac{(1-|a|^2)(1-|x|^2)}{\rho(x,a)}.\end{equation}

Jacobian of the mapping $y=h_a(x)$ is given by  \begin{equation}\label{jaco}J(y,x)=\frac{1-|a|^2}{\rho(a,x)^n}=\frac{(1-|y|^2)^n}{(1-|x|^2)^n}.\end{equation}

\subsection{Bergman ball}
\label{Bergman_ball_subsect}

The norm of a vector $z=(z_1,\dots,z_m) \in \mathbb{C}^m$ is $|z|=\sqrt{\left<z,z\right>}$, where $\left<z,w\right>=\sum_{k=1}^m z_k\overline{w_k}.$

Denote by $\mathrm{d}\lambda(z)$ the Lebesgue measure in $\mathbb{C}^m$. Then the hyperbolic measure reads
\begin{equation} \label{Berg_hyp_meas}
 \mathrm{d}\Lambda(z) =\frac{\mathrm{d}\lambda(z)}{(1-|z|^2)^{m+1}}.
\end{equation}

\begin{definition}
Let $\mathbb{D}^m$ be the unit ball in $\mathbb{C}^m$ equipped with the metric tensor defined by $
g_z(u,v) = \left<B(z)u,v\right>$, $u, v\in \mathbb{C}^m$, $z\in \mathbb{D}^m.$
Here $$B(z)=(b(z)_{ij})_{i,j=1}^n \quad \mbox{ and } \quad b(z)_{ij}= \frac{1}{n+1}\frac{\partial^2}{\partial \overline{{z_i}}\partial z_j}K(z,z),$$ where $$K(z,w)=\frac{1}{n+1}\frac{1}{(1-\left<z,w\right>)^{n+1}}$$ is the Bergman kernel \cite{kezu}.

The Riemannian manifold $(\mathbb{D}^m, g)$ is named the Bergman  ball.
\end{definition}

It is known that the Bergman balls have constant negative sectional curvature, see \cite{HL}.

Let $P_a$ be the
orthogonal projection of $\mathbb{C}^m$ onto the subspace $[a]$ generated by $a$, and let $$Q=Q_a =
I - P_a$$ be the projection onto the orthogonal complement of $[a]$. Explicitly, $P_0 = 0$ and  $P=P_a(z) =\frac{\left<z,a\right> a}{\left<a, a\right>}$. Set $s_a = (1 - |a|^2)^{1/2}$ and consider the map
\begin{equation}
\label{Bergman_transf}
m_a(z) =\frac{a-P_a z-s_a Q_a z}{1-\left<z,a\right>}.
\end{equation}
Compositions of mappings of the form \eqref{Bergman_transf} and unitary linear mappings in $\mathbb{C}^m$ consitute the group of holomorphic automorphisms of the unit ball $\mathbb{D}^m \subset \mathbb{C}^m$. It is easy to verify that $m_a^{-1}=m_a$. Moreover, for any automorphism $q$ of the Bergman ball onto itself there exists a unitary transformation $U$ such that \begin{equation}\label{automob} m_{q(c)}\circ q=U\circ m_c.\end{equation}

By using the representation formula \cite[Proposition~1.21]{kezu}, we can introduce the Bergman metric by the following formula \begin{equation}\label{bergmet}d_B(z,w)=\frac{1}{2}\log\frac{1+|m_w(z)|}{1-|m_w(z)|}.\end{equation}
If $\Omega= \{z\in \mathbb{C}^n:\left<z,a\right>\neq 1\}$,  then the map $p_a$  is holomorphic in $\Omega$.

It is well-known that every automorphism $q$ of the unit ball is an isometry w.r. to the Bergman metric, that is: $d_B(z,w)=d_B(q(z),q(w))$.

We also point out the formula \begin{equation}\label{phia}(1-|m_a(z)|^2)=\frac{(1-|z|^2)(1-|a|^2)}{|1-\left<a,z\right>|^2}\end{equation} and the expression for the Jacobian
$$
J(z,m_a)=\left(\frac{1-|m_a(z)|^2}{1-|z|^2}\right)^{n+1}=\left(\frac{1-|a|^2}{|1-\left<z,a\right>|^2}\right)^{n+1}.
$$

\section{Conformal barycenter in the Poincar\' e disc}\label{sec:3}

The concept of the mean in the Poincar\' e disc can be studied as a particular issue within a more general context of {\it Riemannian centers of mass} or {\it Karcher means} on manifolds. We refer to the study \cite{ABY} on medians and means on Riemannian manifolds. We also point out the book \cite{Bacak} contains the exposition about means and optimization on Hadamard spaces. Notice that hyperbolic balls are Hadamard manifolds.
Means in hyperbolic spaces and their computation for ML purposes have been discussed in \cite{LKJBLS}.

In this Section we answer the questions $i)$ and $ii)$ from Introduction for the case when $X$ is the Poincar\' e disc. We define the mean (barycenter) as a minimum of a certain potential function, as recently introduced in \cite{JacKal}. We point out favorable properties and present the method of finding the mean by applying the (hyperbolic) gradient descent algorithm for the potential function.

Although computations of the mean in the Poincar\' e disc are tractable and pretty efficient, they are more demanding than in Euclidean spaces. This is the computational cost of hyperbolic representations. We believe that experiments will demonstrate that this cost is acceptable in many setups.

\subsection{Definition and properties}
\label{Conf_bary_Poin_disc}

We consider finite sets (configurations) of points in the Poincar\' e disc. Points are represented by complex numbers $z_1,\dots,z_N$, such that $|z_i|<1$.

\begin{definition}
The configuration $\{z_1,\dots,z_N\}$ is said to be {\it balanced} if $z_1+\cdots+z_N=0$.
\end{definition}

We proceed with several facts and notions that have been introduced in \cite{JacKal}. We will omit the proofs, as they are provided therein.

\begin{theorem}
\label{unique_Mob}
Let $\{z_1,\dots,z_N\}$ be a configuration of points in $\mathbb{B}^2$. Then, there exists a unique (up to a rotation) M\" obius transformation $g_a \in \mathbb{G}_2$, such that the configuration $\{g_a(z_1),\dots,g_a(z_N)\}$ is balanced.
\end{theorem}

In Figure \ref{fig1} we illustrate Theorem \ref{unique_Mob} by plotting three configurations and corresponding balanced configurations.

\begin{figure*}[h]
\centering
  \begin{tabular}{@{}ccc@{}}
    \includegraphics[width=.25\textwidth]{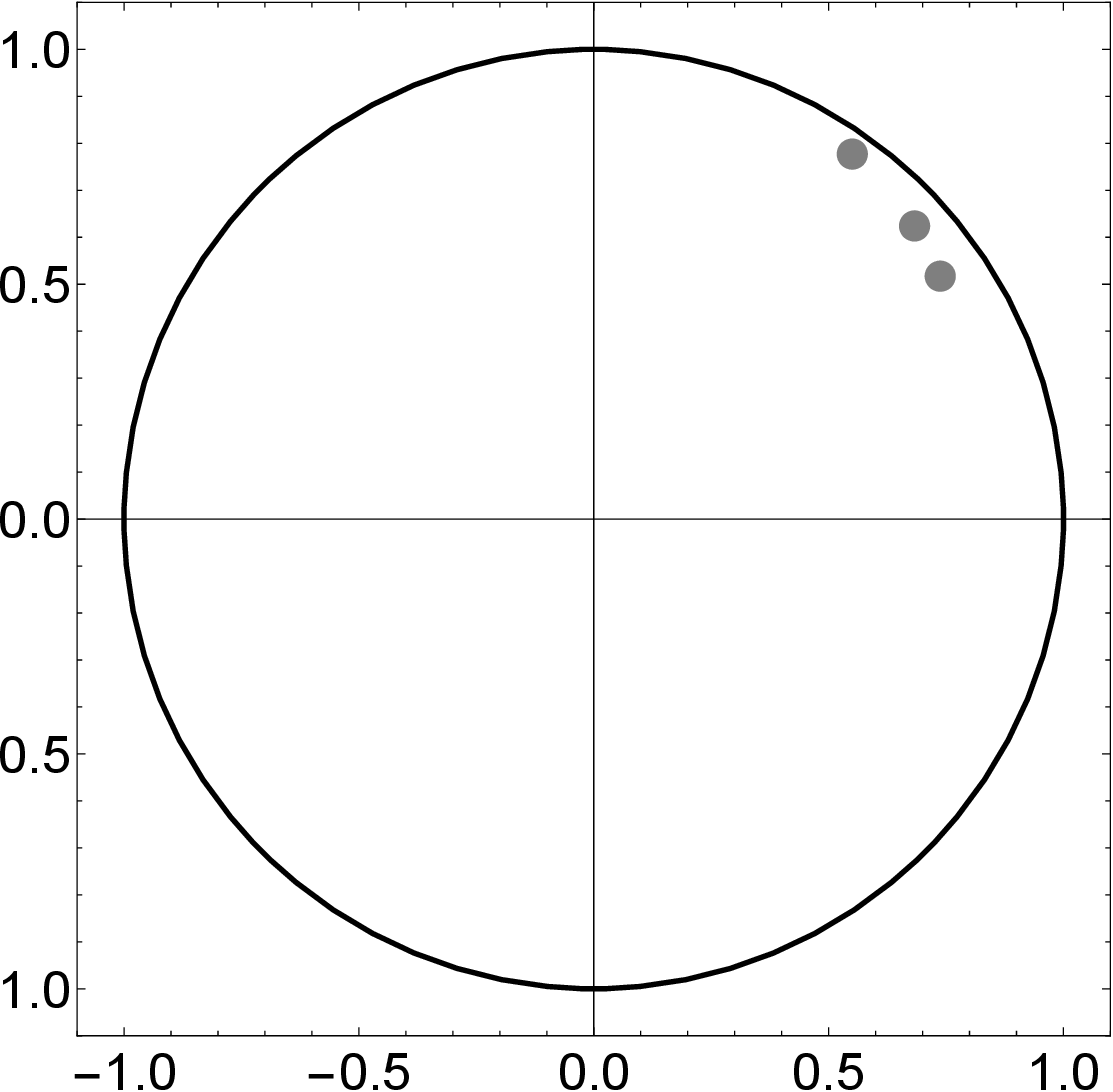}&&\includegraphics[width=.25\textwidth]{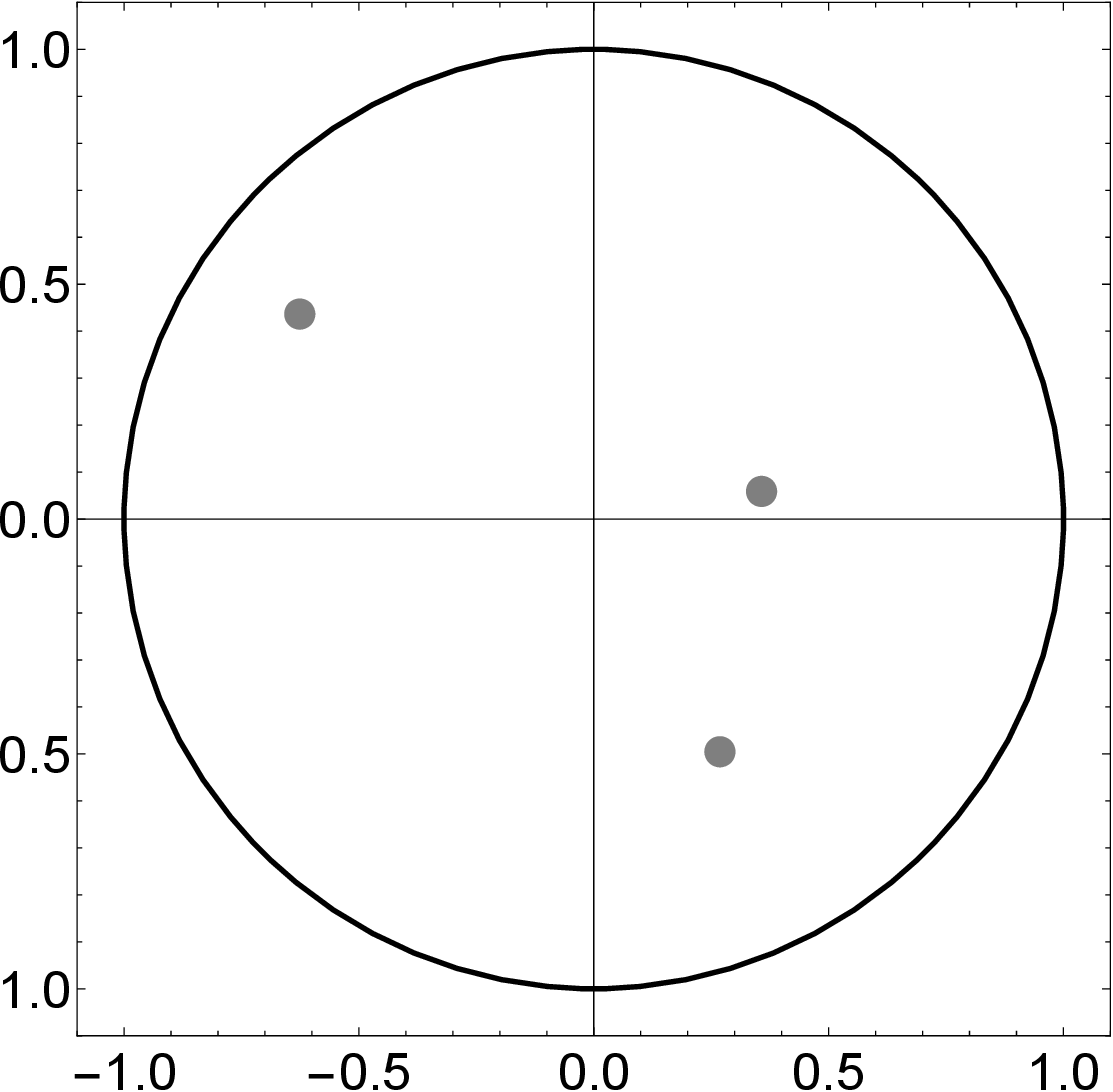}\\
    &a)&\\
     \includegraphics[width=.25\textwidth]{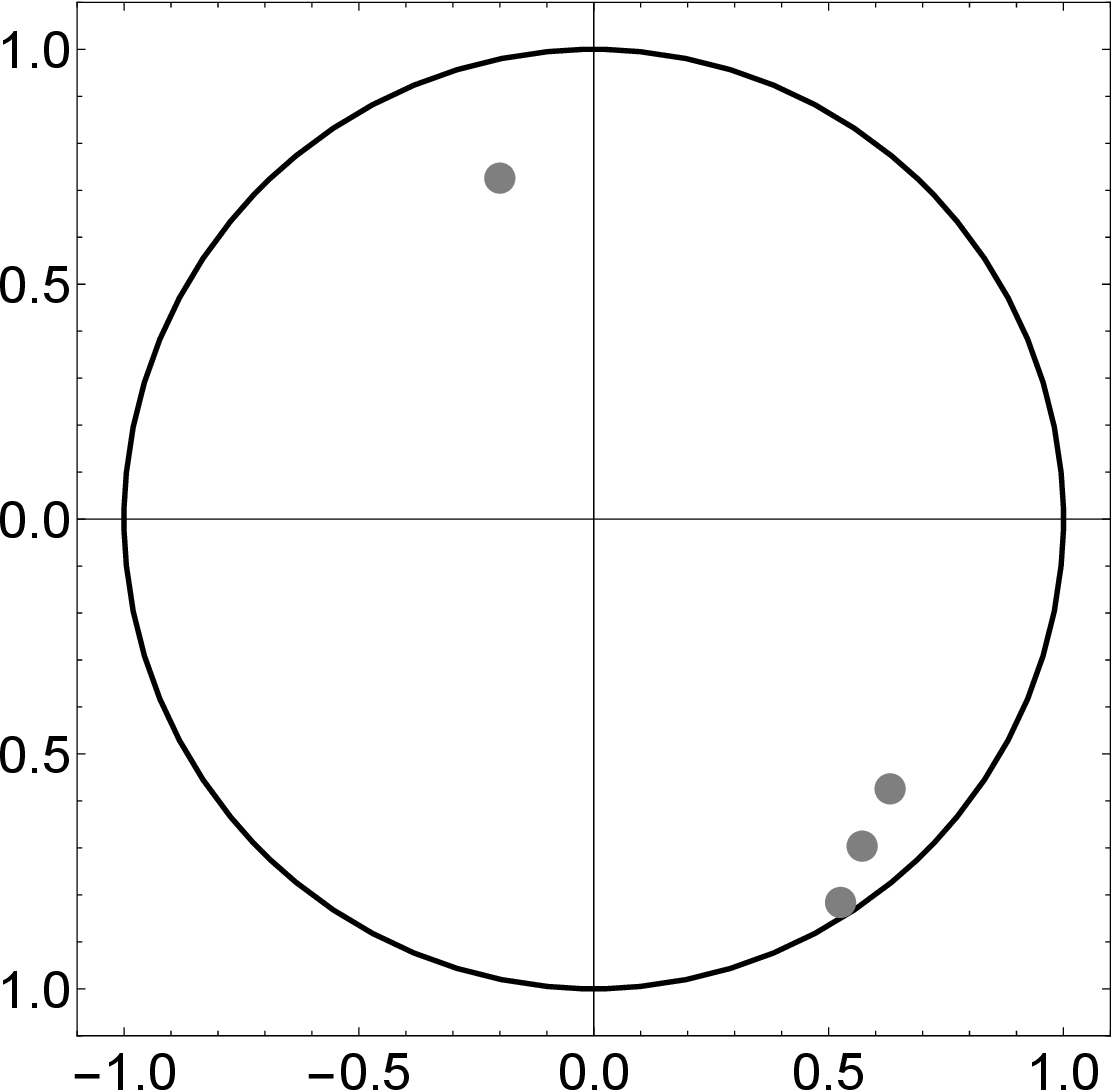}&&\includegraphics[width=.25\textwidth]{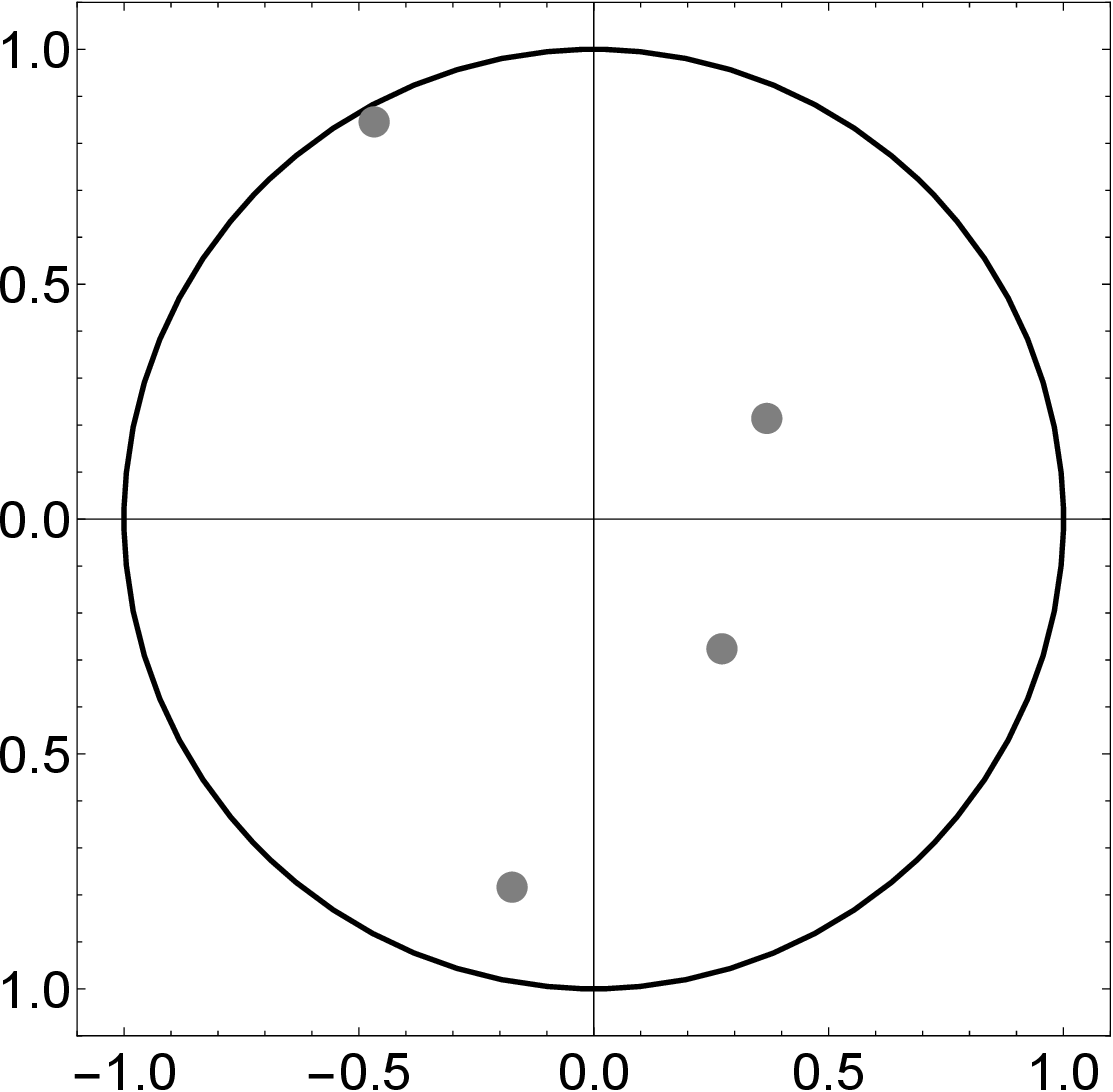}\\
    &b)&\\
     \includegraphics[width=.25\textwidth]{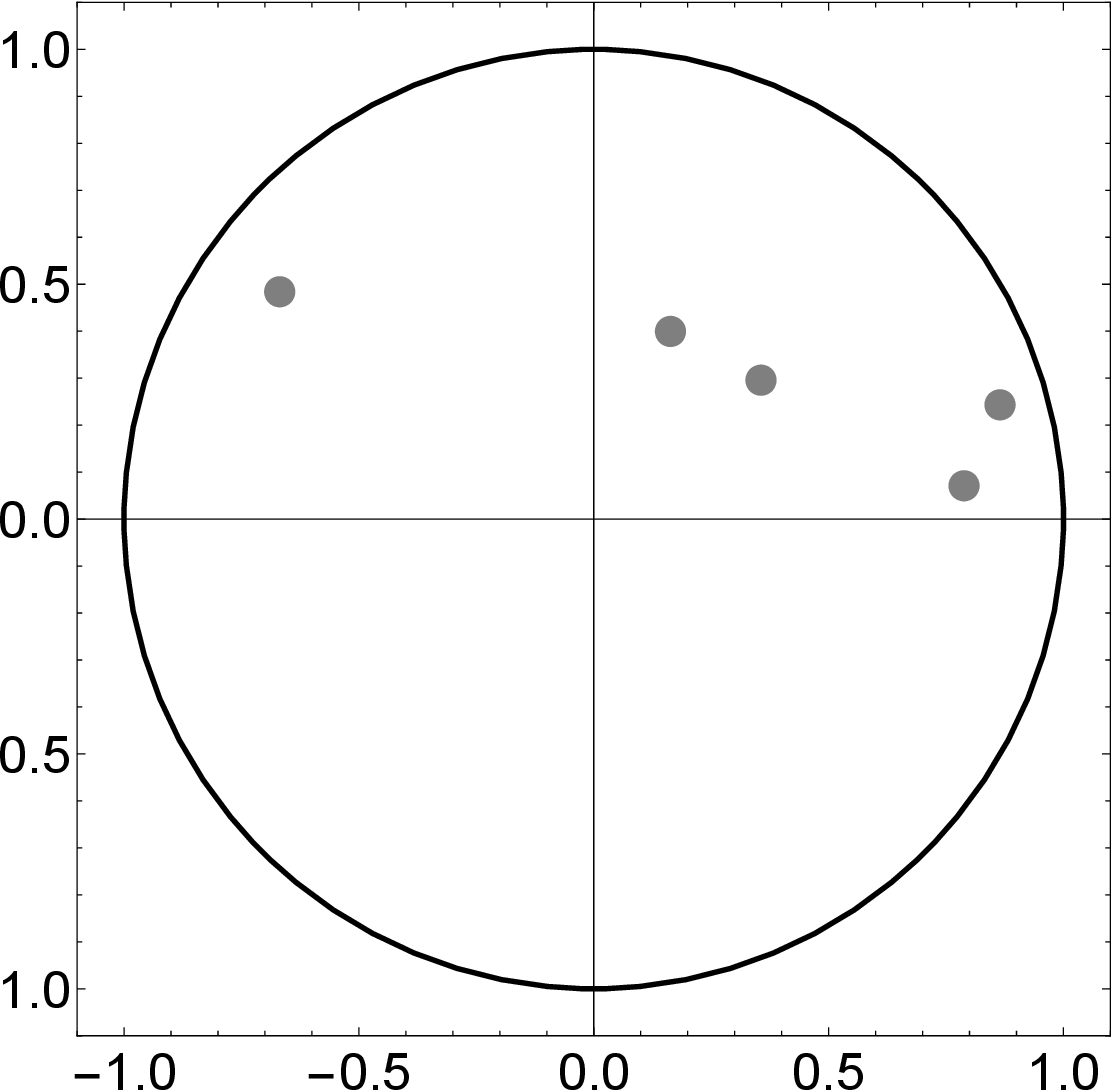}&&\includegraphics[width=.25\textwidth]{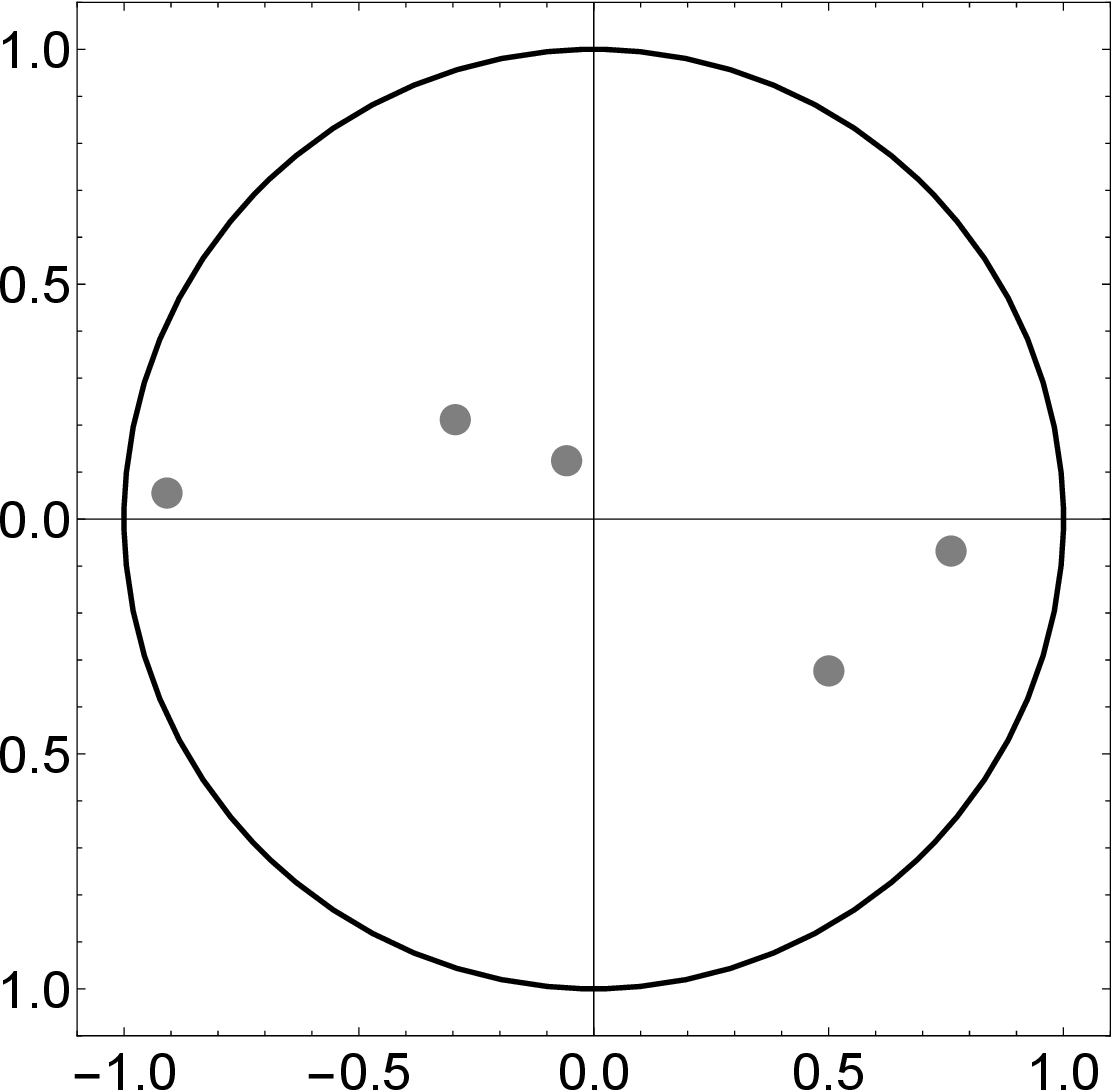}\\
    &c)&\\
  \end{tabular}
  \caption{\label{fig1}This Figure illustrates balanced configurations corresponding to configurations consisting of (a) three points; (b) four points and (c) five points.}
\end{figure*}

\begin{definition}
Consider a configuration $\{z_1,\dots,z_N\}$ in $\mathbb{B}^2$. Then according to the previous Theorem there exists a unique point $a \in \mathbb{B}^2$, such that the configuration $\{g_a(z_1),\dots,g_a(z_N)\}$ is balanced, where $g_a$ is defined as \eqref{Mobius}. The point $a$ is called the {\it conformal barycenter} of the configuration $\{z_1,\dots,z_N\}$.
\end{definition}

\begin{proposition}
Conformal barycenter is conformally invariant, meaning that if $a$ is conformal barycenter of the configuration $\{z_1,\dots,z_N\}$, then for any $h \in \mathbb{G}_2$ the point $h(a)$ is conformal barycenter of the configuration $\{h(z_1),\dots,h(z_N)\}$.
\end{proposition}

\begin{proposition}
The configuration $\{z_1,\dots,z_N\}$ is balanced if and only if its conformal barycenter is zero (center of the disc).
\end{proposition}

For the configuration $\{z_1,\dots,z_N \}$ in $\mathbb{B}^2$ introduce the following function
\begin{equation}
\label{potential_disc}
H(a) = - \sum_{i=1}^N \log \frac{(1-|a|^2)(1-|z_i|^2)}{|1-\bar a z_i|^2}, \quad a \in \mathbb{B}^2.
\end{equation}

\begin{theorem} \cite{JacKal}
The function \eqref{potential_disc} is geodesically convex in $\mathbb{B}^2$ and hence has a unique minimum in $\mathbb{B}^2$. This minimum is conformal barycenter of the configuration $\{z_1,\dots,z_N\}$.
\end{theorem}

The above results assert the existence and uniqueness of the conformal barycenter for any (finite) set of points in the unit disc.\footnote{Note that the notion of conformal barycenter is introduced in \cite{JacKal} for any (not necessarily finite) measurable subset of $\mathbb{B}^2$.} In the remaining part of this Section we demonstrate how this barycenter is computed.

\subsection{Computation of the conformal barycenter}
\label{Comp_conf_barycenter}

Let $\{ z_1,\dots,z_N\}$ be a configuration of points in $\mathbb{B}^2$. Theorem \ref{unique_Mob} asserts that there exists a unique (up to a rotation) M\" obius transformation $\tilde g \in \mathbb G_2$, such that the configuration $\{\tilde g(z_1),\dots,\tilde g(z_N)\}$ is balanced. The conformal barycenter $a$ of the configuration $\{ z_1,\dots,z_N\}$ is the point $a = \tilde g^{-1}(0)$.

Therefore, computing the conformal barycenter boils down to inferring a M\" obius transformation which maps the configuration into a balanced one. We proceed with assertions which provide a method for finding this transformation.

Consider the following system of ODE's in the Poincar\' e disc
\begin{equation}
\label{Poin_swarm}
\frac{d\zeta_j}{dt} = -\frac{K}{2N} \left( \sum_{k=1}^N \bar \zeta_k \right) \zeta_j^2 + \frac{K}{2N} \sum_{k=1}^N \zeta_k, \quad j=1,\dots,N,
\end{equation}
where the notion $\bar w$ stands for the conjugate of the complex number $w$.

We will refer to the system \eqref{Poin_swarm} as {\it Poincar\' e swarm}. One can verify that the unit disc is invariant for dynamics \eqref{Poin_swarm}.

\begin{theorem}
\label{Mob_evol_th}
The system \eqref{Poin_swarm} evolves by actions of the M\" obius group $\mathbb{G}_2$. More precisely, there exists a continuous path $a(t)$ in $\mathbb{B}^2$, such that
\begin{equation}
\label{Mob_evol}
\zeta_j(t) = g_{a(t)}(\zeta_j(0)), \mbox{ for } j=1,\dots,N, \; t>0.
\end{equation}
Furthermore, the point $a(t)$ evolves by the following ODE:
\begin{equation}
\label{bary_evol}
\frac{da}{dt} = - \frac{K}{2N} (1-|a|^2) \sum_{k=1}^N g_a(\zeta_j(0)), \mbox{ where } g_a(z) = \frac{z-a}{1-\bar a z}.
\end{equation}
\end{theorem}

\begin{proof}

For the proof of \eqref{Mob_evol} we will refer to the following general fact from the Lie group theory (see, for instance \cite{Olver}):

\begin{theorem}
\label{Lie_general_th}
Let $L$ be a Lie group with linearly independent infinitesimal generators $v_1,\dots,v_m$. Let $v = c_1 v_1 + \dots + c_m v_m$ be a linear combination of infinitesimal generators, where coefficients $c_i$ depend on the time only. Then the trajectory satisfying $\dot z = v$ with initial condition $z(0) = z_0$ evolves as $z(t) = A_t(z_0)$ for a unique family of transformations $A_t \in L$ parametrized by $t$.
\end{theorem}

Therefore, in order to prove \eqref{Mob_evol} we need to compute infinitesimal generators of the group $\mathbb{G}_2$. Since this is three-dimensional Lie group, the corresponding Lie algebra is three-dimensional vector space.

The general transformation reads $g(z) = e^{i \psi}\frac{z-a}{1+\bar a z}$. To compute infinitesimal generators, evaluate time derivatives of three one-parameter curves corresponding to the three real parameters of the group: $\psi, \Re(a)$ and $\Im(a)$. (Here and below notations $\Re$ and $\Im$ stand for real and imaginary part of the complex number.) Each of the three families is computed by setting the two parameters to zero and replacing the remaining parameter with the time variable. This yields three one-parameter families:
$$
m_1(z) = -e^{i t} z, \quad m_2(z) = \frac{t-z}{1-tz}, \quad m_3(z) = \frac{it-z}{1+itz}.
$$
Time derivatives of the above curves evaluated at $t=0$ yield infinitesimal generators of the M\" obius group:
$$
v_1 = - i z, \quad v_2 = z^2 -1, \quad v_3 = i z^2 + i.
$$

Then from Theorem \ref{Lie_general_th} it follows that any system of the form:
\begin{equation}
\label{Riccati}
\dot z = - i \omega z + (z^2 -1) h_1 + (iz^2 + i) h_2, \mbox{ for } z=(z_1,\dots,z_k)
\end{equation}
with real-valued functions $\omega,h_1$ and $h_2$ evolves by actions of the M\" obius group.

We further notice that the Poincar\' e swarm \eqref{Poin_swarm} is the particular case of \eqref{Riccati} obtained by setting
$$
\omega \equiv 0, \quad h_1(\zeta_1,\dots,\zeta_N) = \frac{K}{2N} \sum_{k=1}^N \Re (\zeta_k), \quad h_2 (\zeta_1,\dots,\zeta_N) = \frac{K}{2N} \sum_{k=1}^N \Im (\zeta_k).
$$
In other words, the right hand side of \eqref{Poin_swarm} is the linear combination of generators of the group $\mathbb{G}_2$. Hence, \eqref{Mob_evol} follows from Theorem \ref{Lie_general_th}.

For the proof of \eqref{bary_evol} differentiate the equality
\begin{equation}
\label{zeta_j}
\zeta_j(t) = \frac{\zeta_j(0)-a(t)}{1-\bar a(t) \zeta_j(0)}
\end{equation}
to obtain
\begin{eqnarray*}
\dot \zeta_j(t) &=& - \frac{\dot a(t)}{1-\bar a(t) \zeta_j(0)} + \dot {\overline a}(t) \zeta_j(0) \frac{\zeta_j(0)-a(t)}{(1-\bar a(t) \zeta_j(0))^2} \nonumber\\
&=&- \frac{\dot a(t)}{\zeta_j(0) - a(t)} \zeta_j(t) + \frac{\dot{\overline a}(t) \zeta_j(0)}{\zeta_j(0)- a(t)} \zeta_j(t)^2 \nonumber\\
&=&-\frac{\dot a(t)}{\zeta_j(0) - a(t)} \zeta_j(t) + \dot{\overline a}(t) \left( 1 + \frac{a(t)}{\zeta_j(0) - a(t)} \right) \zeta_j(t)^2.
\end{eqnarray*}

Further, rearrange \eqref{zeta_j} to get
$$
\frac{1}{\zeta_j(0)-a(t)} = \frac{1 - \bar a(t) \zeta_j(t)}{\zeta_j(t)(1-|a(t)|^2)}.
$$
By substituting this equality in the above ODE we get
$$
\dot \zeta_j(t) = - \frac{\dot a(t)}{1 - |a(t)|^2} + \frac{\dot{\overline a}(t) a(t) - \dot a(t) \bar a(t)}{1 - |a(t)|^2} \zeta_j(t) + \frac{\dot{\overline a}(t)}{1-|a(t)|^2} \zeta_j(t)^2.
$$
Comparison of the last ODE with \eqref{Poin_swarm} yields
$$
- \frac{\dot a(t)}{1-|a(t)|^2} = \frac{K}{2N} \sum_{k=1}^N \zeta_k(t), \; \frac{\dot{\overline a}(t)a(t)-\dot a(t) \bar a(t)}{1-|a(t)|^2} = 0, \; \frac{\dot{\overline a}(t)}{1-|a(t)|^2} = - \frac{K}{2N} \sum_{k=1}^N \bar \zeta_k(t).
$$
From the last three equalities it follows that
$$
\dot a(t) = - \frac{K}{2N}(1-|a(t)|^2) \sum_{k=1}^N \zeta_k(t).
$$
By substituting \eqref{Mob_evol} into the above ODE we obtain \eqref{bary_evol}. This completes the proof.
\end{proof}

\begin{theorem}
\label{hyp_flow_th}
ODE \eqref{bary_evol} with initial conditions $\zeta_j(0)=z_j, \; j=1,\dots,N$ is the gradient flow for the potential \eqref{potential_disc} in hyperbolic metric.
\end{theorem}

Before proving this Theorem we first verify the auxiliary

\begin{lemma}
The ODE $\dot w = f(w) = U(u,v) + iV(u,v)$ in the unit disc is the gradient flow in hyperbolic metric if and only if $f$ satisfies the following condition
\begin{equation}
\label{hyp_grad}
\Im \left( \frac{\partial}{\partial w} [(1-|w|^2)^{-2} f(w)] \right) = 0.
\end{equation}
\end{lemma}

\begin{proof}

Let $\triangle$ be unit disc in the complex plane and $w = u + i v \in \triangle$. Let $h$ be a smooth function on $\triangle$. Using the formalism of Wirtinger derivatives, partial derivatives of $h$ are defined as
$$
\frac{\partial h}{\partial w} = \frac{1}{2} \left( \frac{\partial h}{\partial u} - i \frac{\partial h}{\partial v} \right) \mbox{  and  } \frac{\partial h}{\partial {\bar w}} = \frac{1}{2} \left( \frac{\partial h}{\partial u} + i \frac{\partial h}{\partial v} \right).
$$
Then the Euclidean gradient of a real function $h$ can be written in the complex form as
$$
\nabla_{Eucl} h = 2 \frac{\partial h}{\partial {\bar w}}.
$$
Let $ds$ be standard Euclidean metric in $\mathbb{R}^2$. Then the gradient with respect to a conformal metric $\phi ds$ in $\triangle$ is given by $\phi^{-2} \nabla_{Eucl} h.$ In particular, hyperbolic gradient in $\triangle$ reads
\begin{equation}
\label{hyp_grad1}
\nabla_{hyp} h = \frac{1}{4}(1-|w|^2)^2 \nabla_{Eucl} h = \frac{1}{2} (1-|w|^2)^2 \frac{\partial h}{\partial {\bar w}}.
\end{equation}
Let $f(w) = U + i V$. Then
$$
\frac{df}{dw} = \frac{1}{2} \left[\frac{\partial U}{\partial u} + \frac{\partial V}{\partial v} + i \left(\frac{\partial V}{\partial u} - \frac{\partial U}{\partial v} \right) \right]
$$
and ODE $\dot w = f(w)$ is the Euclidean gradient flow if
$$
\Im \left( \frac{\partial f}{\partial w} \right) = 0.
$$
Accordingly, $\dot w = f(w)$ is the gradient flow in hyperbolic metric if condition \eqref{hyp_grad} holds.
\end{proof}

Now, suppose that the system $\dot w = f(w)$ satisfies the condition \eqref{hyp_grad}. Then there exists the function $h$ on $\triangle$, such that $f = \nabla_{hyp} h$. Then, along trajectories of $\dot w = f(w)$, the function $h$ satisfies the following condition:
\begin{equation}
\label{potent_cond}
\frac{d}{dt} h(w(t)) = |\nabla_{hyp} h(w)|^2_{hyp} = (1-|w|^2)^2 \left| \frac{\partial h}{\partial \bar w} \right|^2.
\end{equation}

We pass to the proof of Theorem \ref{hyp_flow_th}.

\begin{proof}
Let $z_1,\dots,z_N$ be points in $\mathbb{B}^2$ and consider the ODE \eqref{bary_evol}.
It is straightforward to verify that the hyperbolic gradient condition \eqref{hyp_grad} holds for this ODE.

We further determine the potential function for (\ref{bary_evol}). Denote this function by $H(a)$. Comparing \eqref{bary_evol} and \eqref{hyp_grad1} we find that

\begin{eqnarray*}
\frac{\partial H}{\partial \bar a} &=& - (1-|a|^2)^{-1} \frac{1}{N} \sum \limits_{j=1}^N g_a(z_j) = -\frac{1}{N} \frac{1}{1-\bar a a} \sum \limits_{j=1}^N \frac{z_j - a}{1-\bar a z_j} \\
&=& \frac{1}{N} \sum \limits_{j=1}^N \left( \frac{a}{1-\bar w w} - \frac{z_j}{1 - \bar a z_j} \right).
\end{eqnarray*}

Integrating this expression with respect to $\bar a$ (and treating $a$ as a constant) we get:
\begin{equation}
\label{H1}
H(a) = \frac{1}{N} \sum \limits_{j=1}^N \log \left( \frac{1-z_j \bar a}{1 - a \bar a} \right) + g(a).
\end{equation}
Here, $g(a)$ is an integration constant. We can choose it in such a way to make $H$ real and bounded. To that aim set:
$$
g(a) = \frac{1}{N} \sum \limits_{j=1}^N \log \left( \frac{1 - \bar z_j a}{1 - \bar z_j z_j} \right).
$$
Plugging this into (\ref{H1}) yields:
$$
H(a) = - \frac{1}{N} \sum \limits_{j=1}^N \log \frac{(1-|a|^2)(1-|z_j|^2)}{|1-z_j \bar a|^2}.
$$
This completes the proof.
\end{proof}

Theorems \ref{Mob_evol_th} and \ref{hyp_flow_th} imply the hyperbolic gradient descent algorithm for minimization of \eqref{potential_disc}. Given the configuration $\{z_1,\dots,z_N\}$ consider the Poincar\' e swarm \eqref{Poin_swarm} with $K<0$ and initial conditions
$$
\zeta_1(0) = z_1,\dots,\zeta_N(0)=z_N.
$$
Then, for sufficiently large $T$, the configuration $\{ \zeta_1(T),\dots,\zeta_N(T)\}$ is approximately balanced and there exists $g_a$, such that $g_a(z_j) = \zeta_j(T)$ for $j=1,\dots,N$. Then $g_a^{-1}(0)$ is conformal barycenter of the configuration $\{ z_1,\dots,z_N\}$.

In Figure \ref{fig2} we plot conformal barycenters for several configurations in $\mathbb{B}^2$, computed via Poincar\' e swarming dynamics \eqref{Poin_swarm}.

\begin{figure*}[h]
\centering
  \begin{tabular}{@{}ccc@{}}
    \includegraphics[width=.25\textwidth]{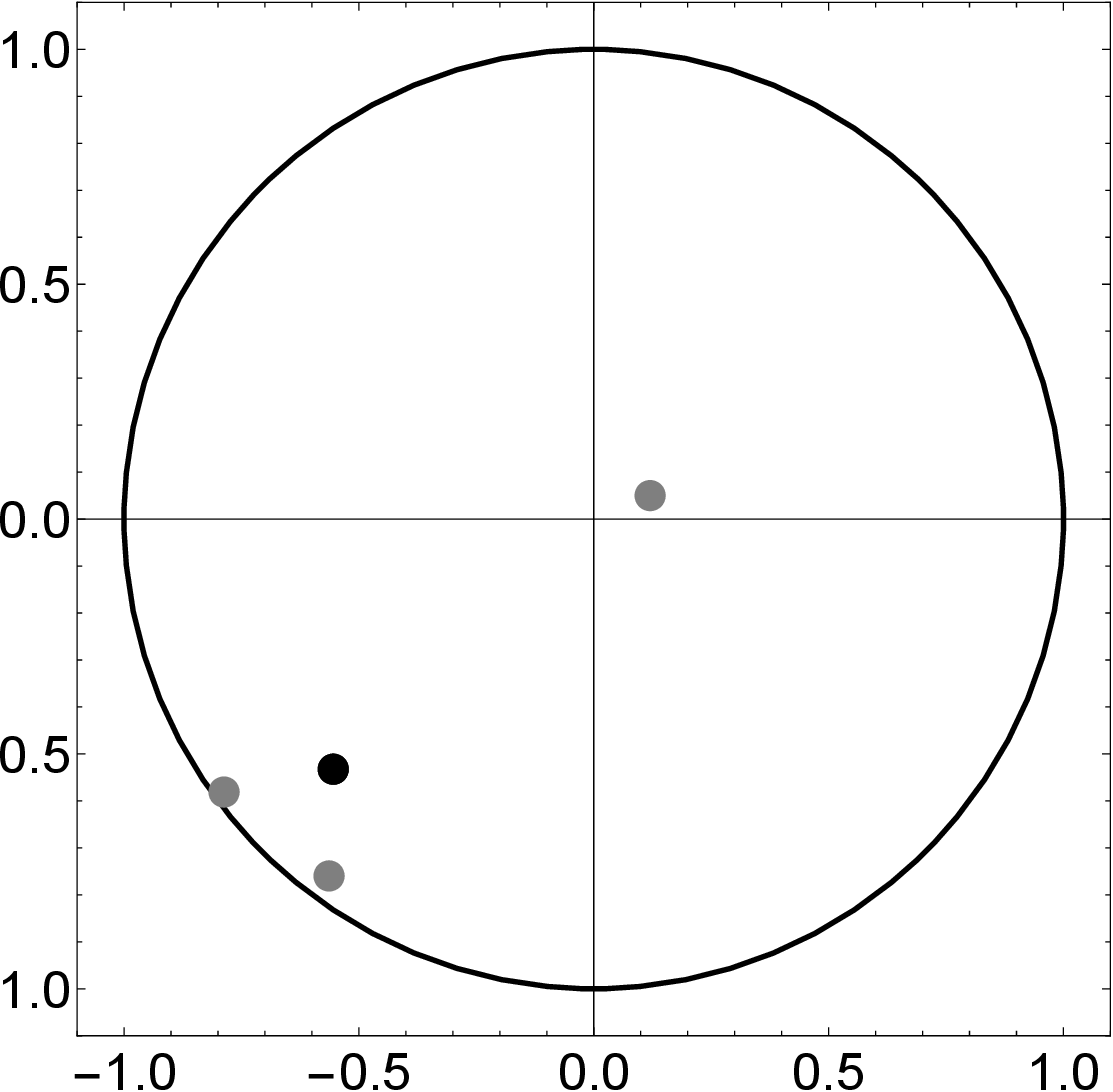}&\includegraphics[width=.25\textwidth]{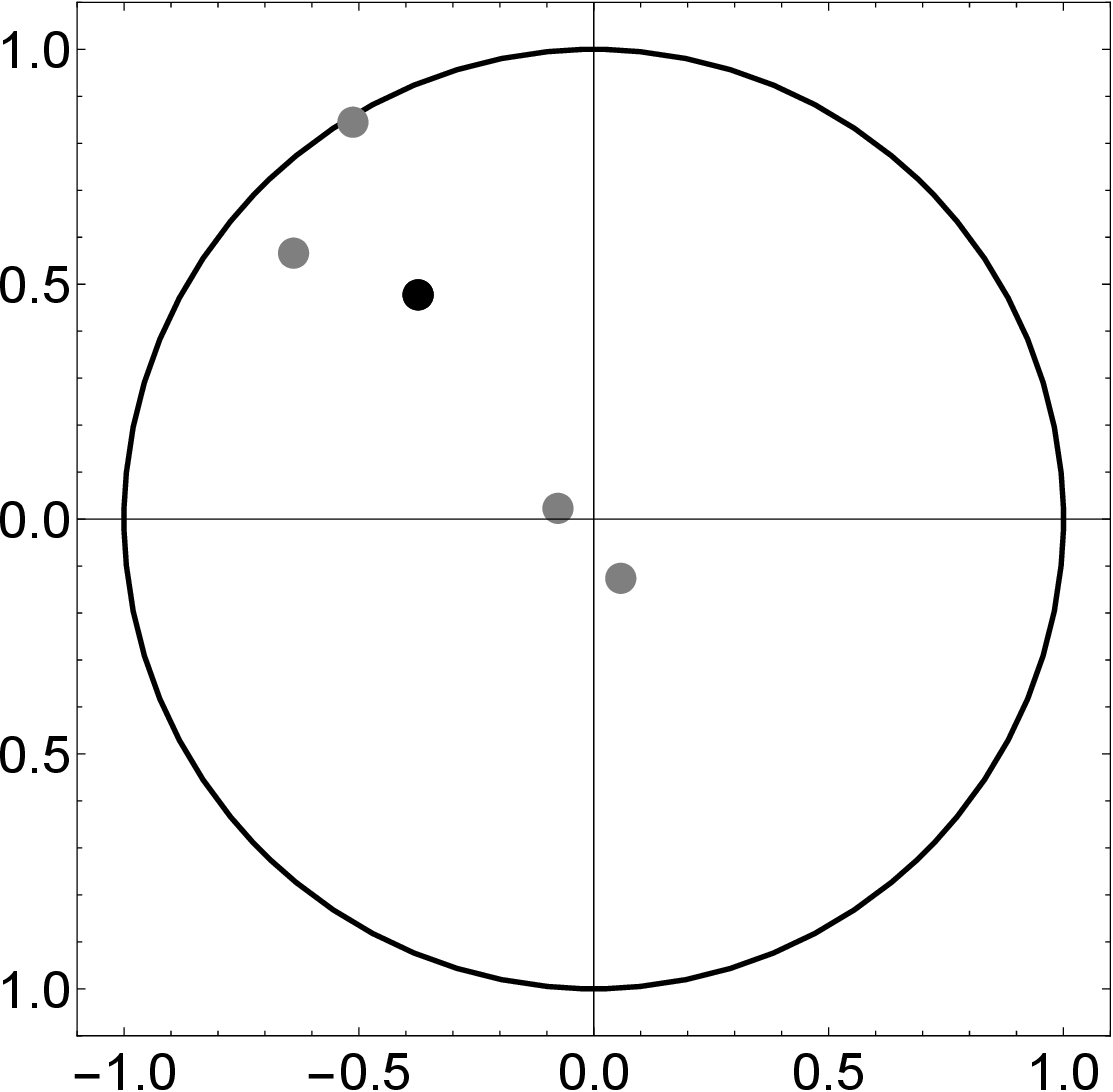}&\includegraphics[width=.25\textwidth]{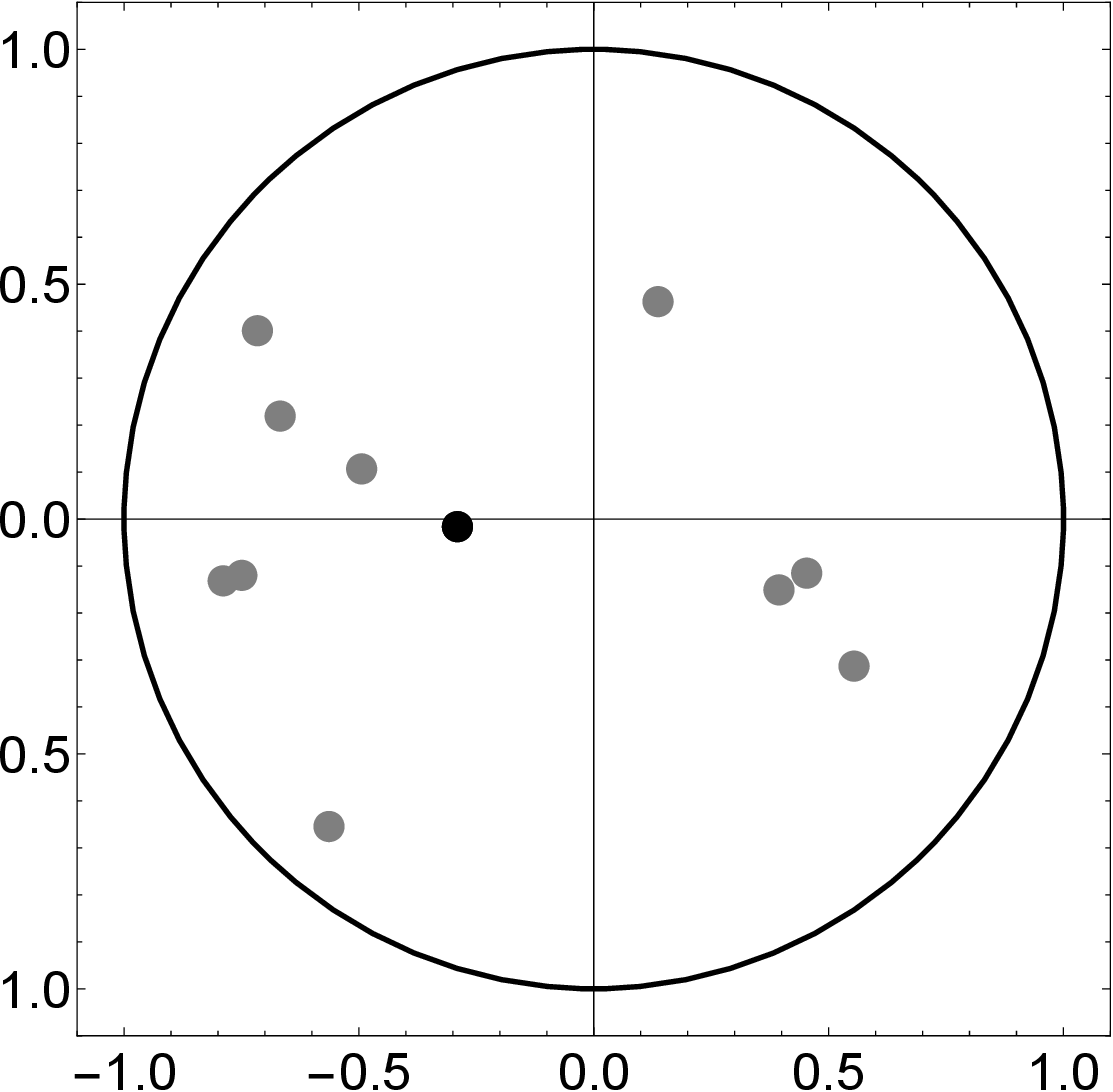}\\
    a)&b)&c)\\
  \end{tabular}
  \caption{\label{fig2}Barycenters in the Poincar\' e disc.}
\end{figure*}

\begin{remark}
Results exposed in this Section are to a great extent inspired by geometric investigations of Kuramoto ensembles of coupled phase oscillators \cite{Kuramoto,MMS}. Phase oscillators are described by their phases (angles) only, while amplitudes are neglected. Therefore, Kuramoto model with $N$ oscillators can be treated as dynamical system on the $N$-torus $\mathbb{S}^1 \times \cdots \times \mathbb{S}^1$. In the continuum limit this model can be considered as a dynamical system on the space of probability measures on the circle.

It has been demonstrated in \cite{MMS} that oscillators in the Kuramoto model with global coupling evolve by actions of the M\" obius group. This result unveiled the group-theoretic background of the low-dimensional dynamical system for global variables in such models previously derived in \cite{WS}. More recently, it has been shown in \cite{CEM} that the dynamics of Kuramoto oscillators on the circle induce hyperbolic gradient flows in the disc. In particular, Kuramoto models with global negative coupling compute conformal barycenter of a measure on the circle by minimizing the Buseman function (see \cite{CEM,Jacimovic} for this result and the classical paper \cite{DE} where the notion of conformal barycenter has been introduced). Finally, the relation between hyperbolic geometry in balls and generalized Kuramoto models on spheres has been exposed in \cite{LMS}. Poincar\' e swarms \eqref{Poin_swarm} introduced in this Section compute conformal barycenters of configurations in $\mathbb{B}^2$ in a similar way as Kuramoto ensembles compute conformal barycenters of configurations on the unit circle $\mathbb{S}^1$. Notice that swarms in the hyperbolic disc (and, more generally, on homogeneous spaces) have previously appeared as gradient systems in studies on geometric consensus theory, see \cite{Sepulchre}.

All the results mentioned in the present remark have been obtained in studies on physics of complex systems (more precisely, studies of the Kuramoto model and its higher-dimensional extensions). In the view of new challenges and recent developments in data representations, these results can be significant for ML in spherical and hyperbolic spaces.
\end{remark}

\section{Statistical modeling in the Poincar\' e disc}
\label{Moeb_Poin_disc}

One of current limitations of hyperbolic ML is the lack of adequate statistical models. So far, majority of experiments \cite{Poleksic,CYPVLL,NYFK} exploited so-called "generalized normal distribution on hyperbolic spaces" for encoding uncertainties. However, inadequacy of this model seems to strongly affect the efficiency of hyperbolic ML algorithms, especially in high dimensions. In the present Section we answer the questions $iii)$ and $iv)$ from Introduction for the case when $X$ is the Poincar\' e disc. The answers are based on a novel family of probability distributions in the Poincar\' e disc, which posses many desirable properties, including conformal invariance, easy random variate generation and efficient estimation of parameters.

\subsection{The M\" obius family of probability distributions in the Poincar\' e disc}

We consider probability distributions on $\mathbb{B}^2$ defined by densities of the form (see \cite{JacMark})
\begin{equation}
\label{conf-nat-disc}
p(x;a,s) = \frac{s-1}{\pi} \left( \frac{(1-|z|^2)(1-|a|^2)}{|1-\bar a z|^2} \right)^s.
\end{equation}
with parameters $a \in {\mathbb B}^2$ and $s>1$. As will be clear from the exposition below, parameters $a$ and $s$ have a transparent meaning. Point $a$ is the conformal barycenter of the measure, as defined in \cite{JacKal}. It coincides with the Riemannian center of mass of the measure and should be treated as the mathematical expectation of the distribution \eqref{conf-nat-disc}. On the other hand, $s$ is the concentration parameter. For higher values of $s$ the samples are more concentrated around the mean point $a$, as illustrated in Figure \ref{fig3} below.

Notice that all densities \eqref{conf-nat-disc} are unimodal and symmetric in hyperbolic metric.\footnote{We mention that analogous family of probability measures over the Euclidean ball has been introduced in \cite{Jones} and named the M\" obius distributions by the author (due to their invariance w. r. to actions of the group of M\" obius transformations). These distributions are asymmetric (skewed) in Euclidean metric.} We denote this family by $Moeb_2(a,s)$ and refer to these distributions as {\it M\" obius distributions in hyperbolic disc.}

Sub-families $Moeb_2(a,s^*)$ with fixed $s^*$ are conformally invariant. More precisely
\begin{proposition}
\label{conf_invariance}
If $\mu \sim Moeb_2(a,s^*)$ then $g_* \mu \sim Moeb_2(g(a),s^*)$ for any $g \in \mathbb{G}_2$, where $g_*$ denotes the pullback measure defined by \eqref{pull_back}.
\end{proposition}
Moreover, the group $\mathbb{G}_2$ acts transitively on $Moeb_2(a,s^*)$.
\begin{proposition}
Let $\mu_1 \sim Moeb_2(a_1,s^*), \, \mu_2 \sim Moeb_2(a_2,s^*)$. Then there exists a M\" obius transformation $g \in \mathbb{G}_2$, such that $g_* \mu_1 = \mu_2$ and $g(a_1)=a_2.$
\end{proposition}

\subsection{Generation of random points in the Poincar\' e disc}

Another advantage of the family $Moeb_2(a,s)$ is the simplicity of generation of the random sample.

$i)$ We first show how to sample a random point from the M\" obius distribution with $a=0$ and arbitrary $s$.

We start by calculating the probability that a random point $z \sim Moeb_2(a=0,s)$ belongs to the smaller ball $\mathbb{B}^2_b$ of radius $b<1$:
\begin{eqnarray}
\label{prob_small_disc}
{\mathbb P}\{z \in \mathbb{B}^2_b\} &=& {\mathbb P}\{|z| < b\} = \int p(z;0,s) d \Lambda(z) \nonumber \\
&=& \frac{s-1}{\pi} \int_{0<|z|<b} (1-|z|^2)^s \frac{d \lambda(z)}{(1-|z|^2)^2} = \frac{s-1}{\pi} \int_0^b (1-|z|^2)^{s-2} d \lambda(z) \\
&=& 1 - (1 - \sqrt{b})^{s-1}.\nonumber
\end{eqnarray}
Recall that $d \lambda(z)$ and $d \Lambda(z)$ denote standard Lebesgue and hyperbolic measure in $\mathbb{B}^2$.

Let $z = |z| e^{i \varphi} \sim Moeb_2(0,s)$. Due to rotational invariance of the distribution, the angle $\varphi$ is uniformly distributed on $[0,2\pi]$. On the other hand, the probability distribution function for $|z|$ is given by \eqref{prob_small_disc} and one can generate $|z|$ by inverting \eqref{prob_small_disc}. To that aim, sample a random number $\kappa$ from the uniform distribution on $[0,1]$ and set
\begin{equation}
\label{generate_r}
|z| = \left( 1 - \sqrt[s-1]{1-\kappa} \right)^2.
\end{equation}

$ii)$ Now, due to the conformal-invariance property stated in Proposition \ref{conf_invariance} we can transform a random point $z_1 \sim Moeb_2(0,s)$ to $z_2 \sim Moeb_2(a,s)$ with arbitrary $a \in \mathbb{B}^2$ by acting on $z_1$ with M\" obius transformation $g_a$, such that $g_a(0) = a$.

Random samples from $Moeb_2(a,s)$ for several values of $a$ and $s$ are depicted in Figure \ref{fig3}.

\begin{figure*}[h]
\centering
  \begin{tabular}{@{}ccc@{}}
    \includegraphics[width=.25\textwidth]{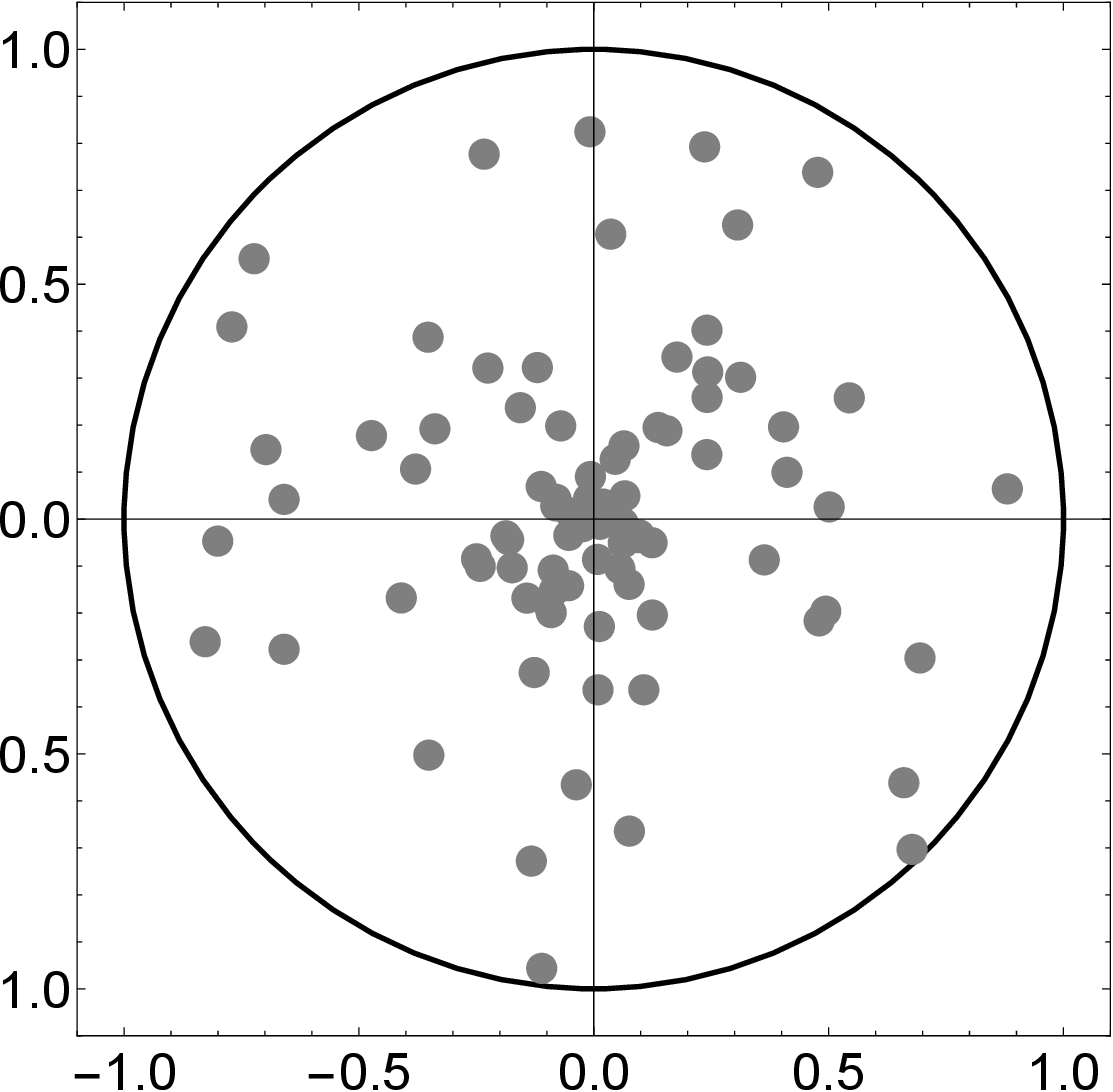}&\includegraphics[width=.25\textwidth]{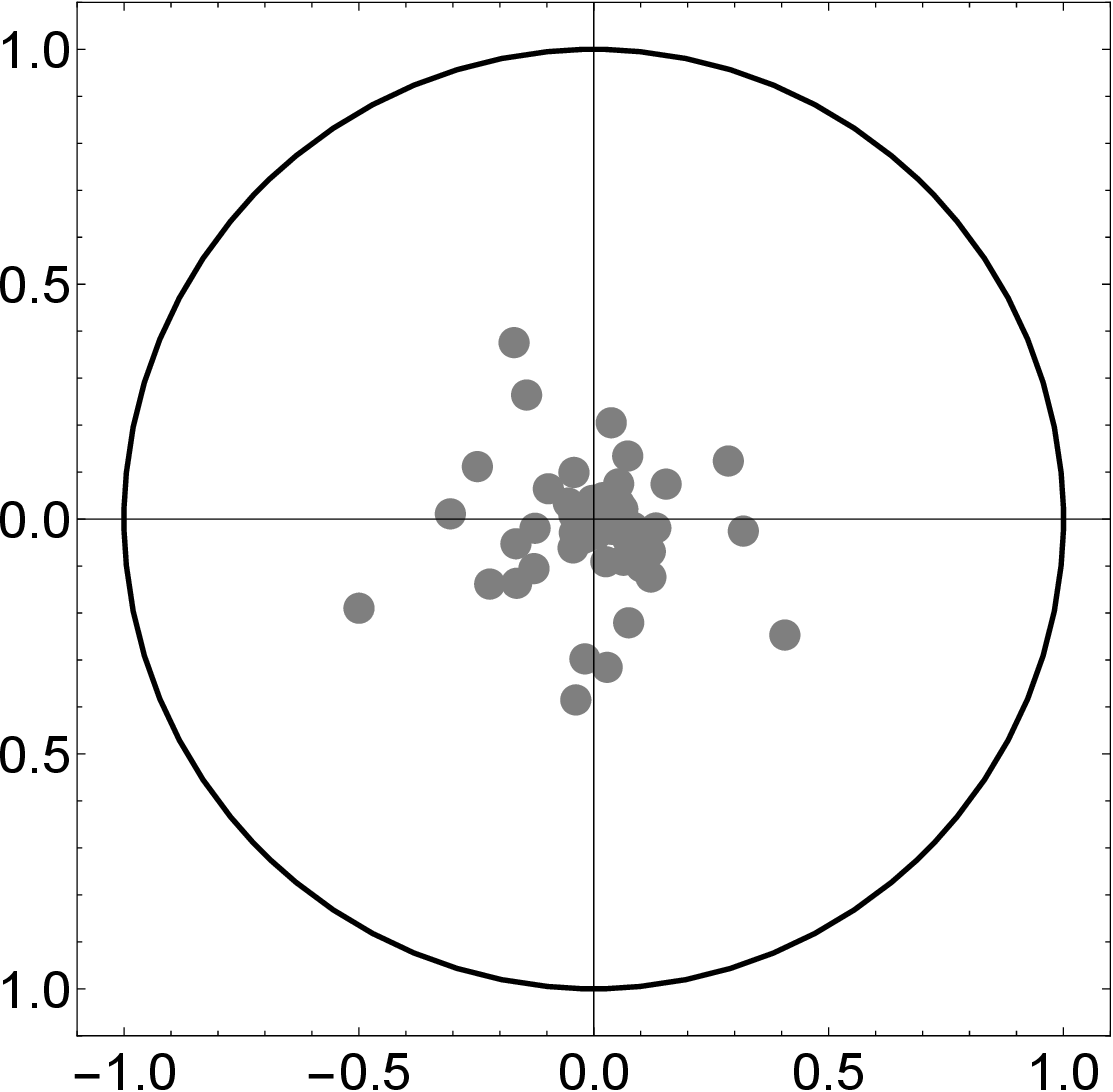}&\includegraphics[width=.25\textwidth]{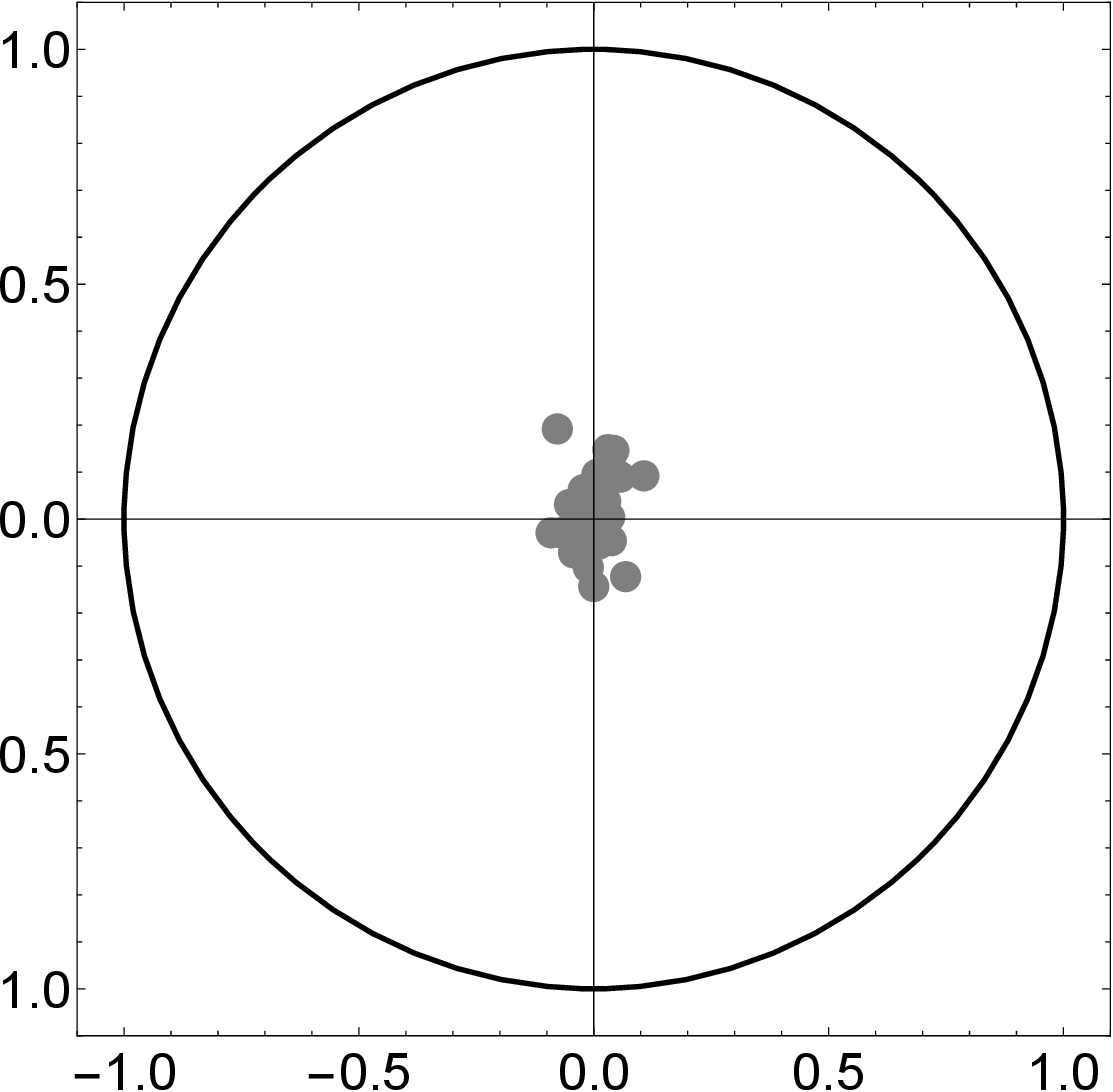}\\
    a)&b)&c)\\
        \includegraphics[width=.25\textwidth]{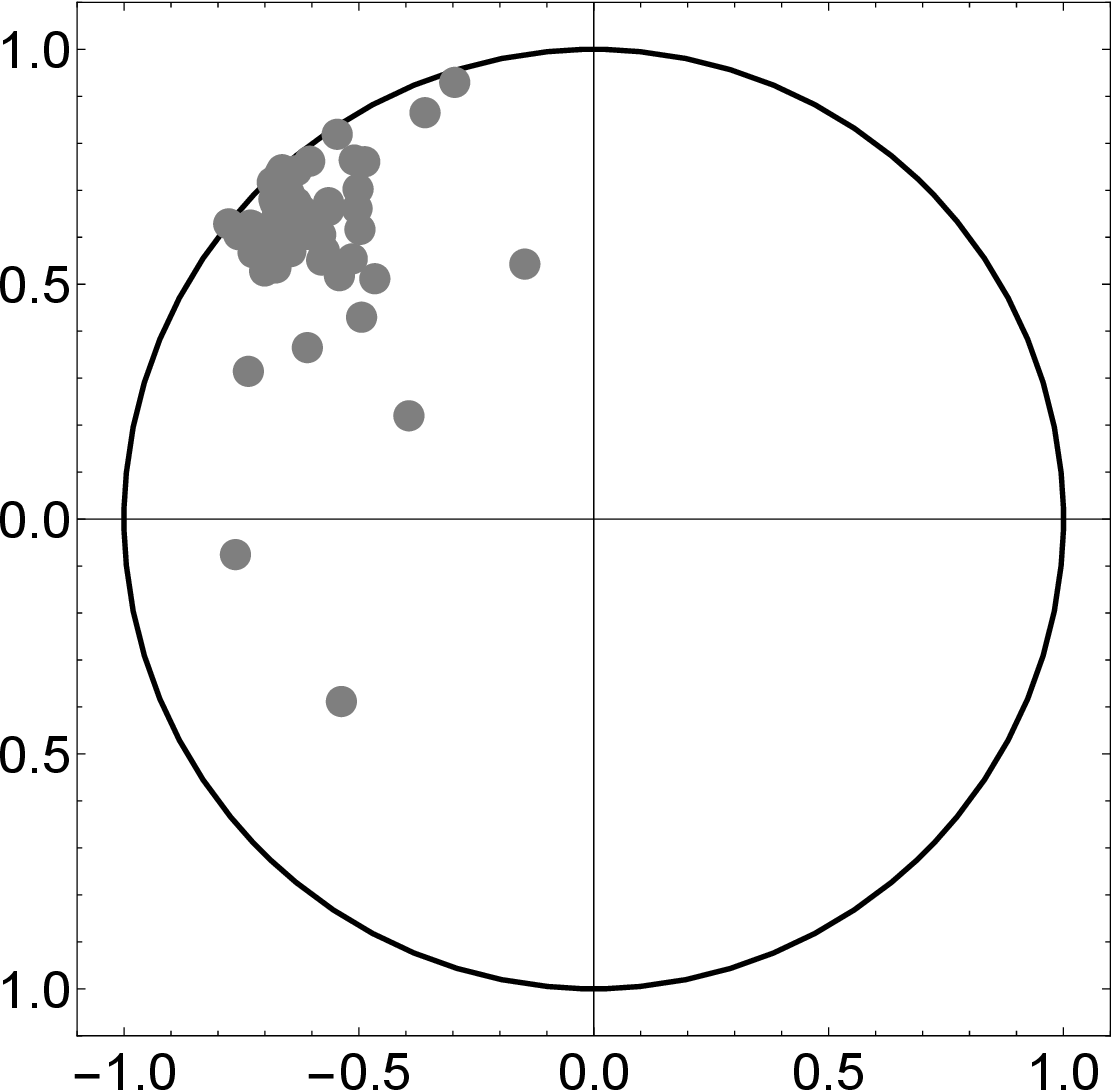}&\includegraphics[width=.25\textwidth]{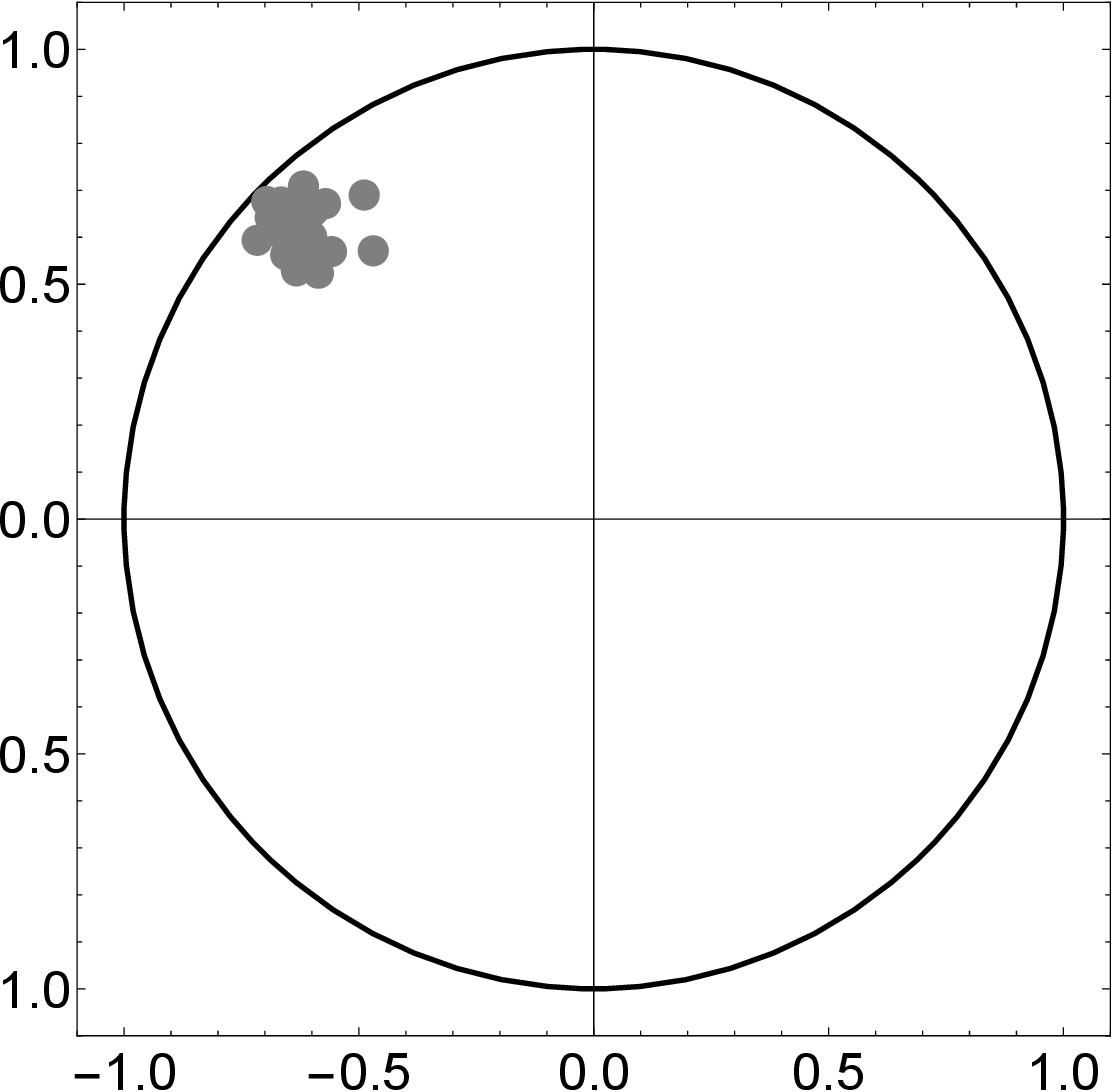}&\includegraphics[width=.25\textwidth]{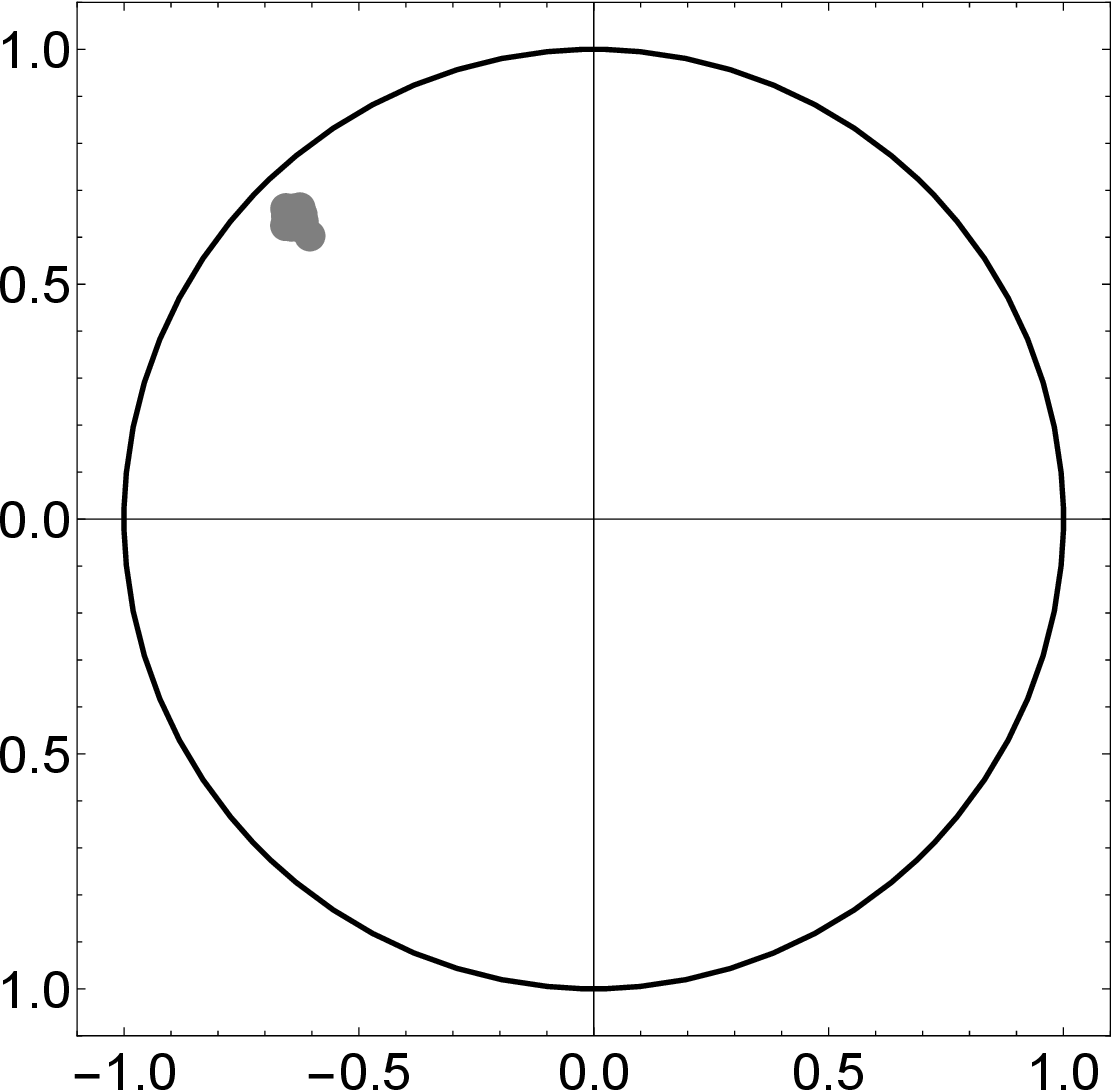}\\
    d)&e)&f)\\
  \end{tabular}
  \caption{\label{fig3}Random samples in the Poincar\' e disc with parameters: a) $a=0$, $s=2$; b) $a=0$, $s=4$; c) $a=0$, $s=8$; d) $a=0.9e^{i\frac{3\pi}{4}}$, $s=2$; e) $a=0.9e^{i\frac{3\pi}{4}}$, $s=4$; f) $a=0.9e^{i\frac{3\pi}{4}}$, $s=8$.}
\end{figure*}

\subsection{Maximum likelihood for the M\" obius family}

We proceed with the maximum likelihood estimation. Given the observations $z_1,\dots,z_N \sim Moeb_2(a,s)$, construct the likelihood function:
\begin{equation}
\label{ML_disc}
L(a,s|z_i) = \prod_{i=1}^N p(z_i;a,s) = \frac{(s-1)^N}{\pi^N} \prod_{i=1}^N \left( \frac{(1-|a|^2)(1-|z_i|^2)}{|1-\bar a z_i|^2} \right)^s.
\end{equation}
Taking the logarithm of $L(a,s|z_i)$ yields
\begin{equation}
\label{log_ML_disc}
\log L(a,s|z_i) = N \log(s-1) - N \log \pi + s \sum_{i=1}^N \log \frac{(1-|a|^2)(1-|z_i|^2)}{|1-\bar a z_i|^2}, \quad a \in \mathbb{B}^2, \, s>1.
\end{equation}

The log-likelihood function \eqref{log_ML_disc} has a unique maximizer $(\hat a,\hat s)$. We ignore the constant $N \log \pi$ and pass to the following maximization problem
\begin{equation}
\label{max_log_likelihood_disc}
\mbox{Maximize } \log L(a,s|z_i) = N \log(s-1) - s H(a),  \mbox{ w. r. to } a \in \mathbb{B}^2, \, s>1,
\end{equation}
where $H(a)$ is given by \eqref{potential_disc}.

It is evident that the optimization problem \eqref{max_log_likelihood_disc} splits into two separate problems for variables $a$ and for $s$. Indeed, the function $H(a)$ is positive, hence maximum of \eqref{max_log_likelihood_disc} is achieved at the minimizer $\hat a$ of $H(a)$. Therefore, $\hat a$ is precisely the conformal barycenter of points $z_1,\dots,z_N$. As shown in subsection \ref{Comp_conf_barycenter} hyperbolic gradient descent algorithm \eqref{Poin_swarm} computes this point.

Once $\hat a$ is determined, $\hat s$ can be found as the solution of the following maximization problem on $(1,+\infty)$:
$$
\hat s = argmax \left( N \log (s-1) - s H(\hat a) \right), \mbox{ for } s>1.
$$
Since $H(\hat a)$ is positive and $|H(\hat a)|<N$, it is easy to verify that the objective function is concave for $s>1$. By setting its derivative to zero we obtain the maximum likelihood estimation for $s$:
\begin{equation}
\label{MLE_s}
\hat s = 1 - N \left( \sum_{i=1}^N \log\frac{(1-|a|^2)(1-|z_i|^2)}{|1-\bar a z_i|^2} \right)^{-1}.
\end{equation}

This completes the maximum likelihood derivation for the family \eqref{conf-nat-disc}.

\section{Barycenters in hyperbolic balls}
\label{Bary_balls}

In this Section we answer the questions $i)$ and $ii)$ from Introduction for the case when $X$ is a hyperbolic ball. To that end, we introduce the notion of barycenter in a hyperbolic ball of arbitrary dimensions and assert the conformal (or holomorphic) invariance property. For a given configuration of points, barycenters are minimizers of certain potential functions, as introduced in \cite{JacKal}. We also present the hyperbolic gradient descent algorithm for computation of barycenters. As pointed out in Introduction, there are two ways of introducing hyperbolic metric in balls in even-dimensional vector spaces. To that end, this Section is divided into two subsections, dealing with barycenters in hyperbolic balls in real and complex vector spaces, respectively.

\subsection{Conformal barycenter in the Poincar\' e ball}
\label{conf_bary_ball}

Definition of the conformal barycenter and all the facts from subsection \ref{Conf_bary_Poin_disc} regarding the disc extend to higher-dimensional balls. Replace the Poincar\' e disc $\mathbb{B}^2$ with the $d$-dimensional Poincar\' e ball $\mathbb{B}^d$, equipped with the metric \eqref{pome}. Accordingly, the group $\mathbb{G}_2$ of disc-preserving M\" obius transformations in the complex plane should be replaced by the group $\mathbb{G}_d$ of ball-preserving M\" obius transformations in $\mathbb{R}^d$. These transformations are defined by formula \eqref{Mobius_ball}. Notice that the group $\mathbb{G}_d$ is isomorphic to the Lorentz group $SO^+(d,1)$.

Points in $\mathbb{B}^d$ are represented by $d$-dimensional real vectors with the norm less than one. Conformal barycenter of the configuration $\{y_1,\dots,y_N\}$ in $\mathbb{B}^d$ is the unique minimizer of the function \cite{JacKal}
\begin{equation}
\label{potential_Poin_ball}
H_d(a) = - \sum_{i=1}^N \log \frac{(1-|a|^2)(1-|y_i|^2)}{\rho(y_i,a)},
\end{equation}
where $\rho(u,v)$ is defined in \eqref{rho}.

Just as in the case of the Poincar\' e disc, conformal barycenter in $\mathbb{B}^d$ is conformally invariant, i.e. if $a$ is conformal barycenter of the configuration $\{y_1,\dots,y_N \}$ and $h \in \mathbb{G}_d$, then $h(a)$ is the conformal barycenter of the configuration $\{h(y_1),\dots,h(y_N) \}$. For the proof of this property we refer to \cite{JacKal}.

In order to compute the conformal barycenter we construct the hyperbolic gradient flow for the potential function \eqref{potential_Poin_ball}. To that aim, compute infinitesimal generators of the group $\mathbb{G}_d$. Infinitesimal generators can be evaluated separately for the boost and the rotation part of the transformation \eqref{Mobius_ball}. We omit the computation and refer to \cite{LMS} for more detailed similar derivations. In such a way we find infinitesimal rotation:
$$
x \to Ax, \mbox{ where } A \mbox{ is an anti-symmetric } d \times d \mbox{ matrix}
$$
and infinitesimal boost:
$$
x \to 2 \langle w,x \rangle x - (1 + |x|^2) w, \mbox{ where } w \in \mathbb{B}^d.
$$

Now, introduce the swarm:
\begin{equation}
\label{Poin_swarm_ball_gen}
\dot x_j = A x_j - K \left \langle x_j, f \right \rangle x_j + \frac{K}{2} (1 - |x_j|^2) f, \quad j=1,\dots,N
\end{equation}
where $f \equiv f(x_1,\dots,x_N)$ is the function taking values in $\mathbb{B}^d$.

Then, applying Theorem \ref{Lie_general_th} for the particular Lie group $\mathbb{G}_d$, we obtain the following
\begin{lemma}
Consider the dynamics \ref{Poin_swarm_ball_gen}.
One has that $x_j(t)=g_t(x_j(0))$ for a unique one-parametric family $g_t \in \mathbb{G}_d$.
\end{lemma}

We further pass to the particular case by setting
 $$
A \equiv 0 \mbox{  and  } f = \frac{K}{N} \sum_{k=1}^N x_i
$$
in \eqref{Poin_swarm_ball_gen}. This yields
\begin{equation}
\label{Poin_swarm_ball}
\dot x_j = - \frac{K}{N} \left \langle x_j, \sum_{k=1}^N x_k \right \rangle x_j + \frac{K}{2N} (1 - |x_j|^2) \sum_{k=1}^N x_k, \quad j=1,\dots,N.
\end{equation}
We call dynamical systems \eqref{Poin_swarm_ball_gen} (including \eqref{Poin_swarm_ball} as a particular case) {\it the Poincar\' e swarms in the hyperbolic ball}.

Therefore, theorems \ref{Mob_evol_th} and \ref{hyp_flow_th} can be extended to hyperbolic balls as substantiated in the following.

\begin{theorem}
The system \eqref{Poin_swarm_ball} evolves by actions of the group $\mathbb{G}_d$. More precisely, there exists a one-parametric family $h_t \in \mathbb{G}_d$, such that
$$
x_j(t) = h_t(x_j(0)), \mbox{ for } j=1,\dots,N, t>0.
$$
Furthermore, the conformal barycenter $a(t)$ of configuration $\{ x_1(t),\dots,x_N(t)\}$ evolves by the following ODE:
\begin{equation}
\label{bary_evol_Poin_ball}
\frac{da}{dt} = \frac{K}{2N}(1-|a(t)|^2) \sum_{i=1}^N h_a(x_i(0)),
\end{equation}
where the transformation $h_a$ is defined by \eqref{Mobius_ball}.
\end{theorem}

\begin{theorem}
ODE \eqref{bary_evol_Poin_ball} is the gradient flow in hyperbolic metric \eqref{pome} for the potential \eqref{potential_Poin_ball}.
\end{theorem}

\begin{proof}
First calculate the Euclidean gradient of \eqref{potential_Poin_ball}:

\begin{eqnarray*}
\nabla_{Eucl} H_d(a) &=& - \sum_{i=1}^N \frac{2 \rho(y_i,a)(1-|y_i)^2)[a \rho(y_i,a) + \rho'(y_i,a)(1-|a|^2)]}{(1-|a|^2)(1-|y_i|^2)\rho(y_i,a)^2} \\
&=&-\frac{2}{1-|a|^2} \sum_{i=1}^N \frac{a|y_i-a|^2 - a(1-|a|^2)+y_i(1-|a|^2)}{|y_i-a|^2+(1-|a|^2)(1-|y_i|^2)} \\
&=&-\frac{2}{1-|a|^2} \sum_{i=1}^N \frac{a|y_i-a|^2 + (y_i-a)(1-|a|^2)}{\rho(y_i,a)} =-\frac{2}{1-|a|^2} \sum_{i=1}^N h_a(x),
\end{eqnarray*}

where $h_a$ are M\" obius transformations in $\mathbb{B}^d$ defined by \eqref{Mobius_ball}.

Then, the relation $\nabla_{hyp} H(a) = \frac{1}{4} (1-|a|^2)^2 \nabla_{Eucl} H(a)$ between the Euclidean and hyperbolic gradients yields
$$
\nabla_{hyp} H_d(a) = \frac{1}{2} (1-|a|^2) \sum_{i=1}^N x_i = \frac{1}{2} (1-|a|^2) \sum_{i=1}^N h_a(x_i(0)).
$$
Therefore, $H_d(a)$ is the potential for \eqref{bary_evol_Poin_ball}.
\end{proof}

In conclusion, conformal barycenter for configuration $\{y_1,\dots,y_N\}$ can be found by solving \eqref{Poin_swarm_ball} with $K<0$ and initial conditions:
$$
x_1(0) = y_1,\dots,x_N(0)=y_N.
$$
Then for some $T$ we have that $\{x_1(t) = h_t(y_1),\dots,x_N(t)=h_t(y_N)\}$ is a balanced configuration whenever $t>T$. The conformal barycenter of $\{y_1,\dots,y_N\}$ is found as the inverse of zero, that is $h^{-1}_t(0)$.

\begin{figure*}[h]
\centering
  \begin{tabular}{@{}cc@{}}
    \includegraphics[width=.3\textwidth]{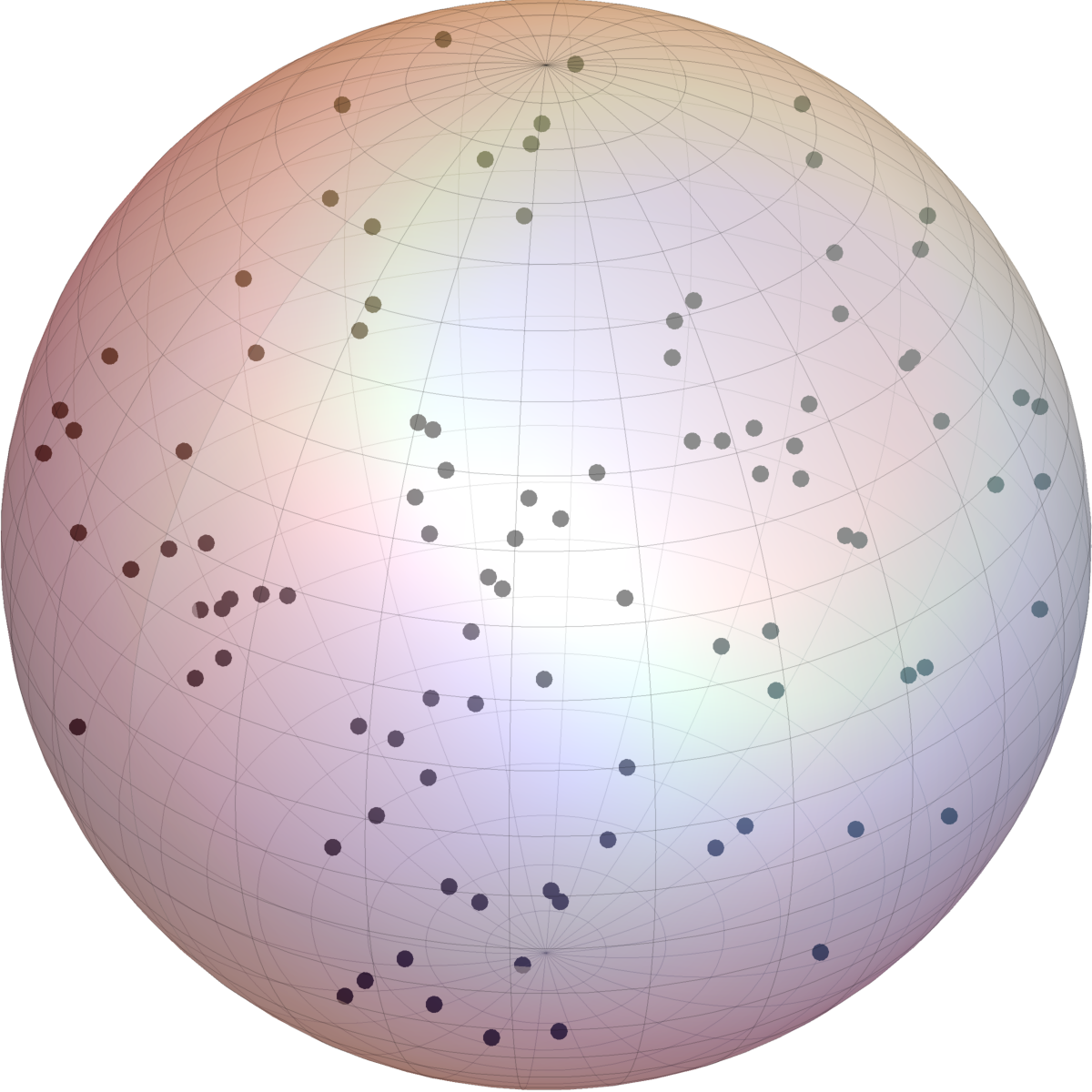}&\includegraphics[width=.3\textwidth]{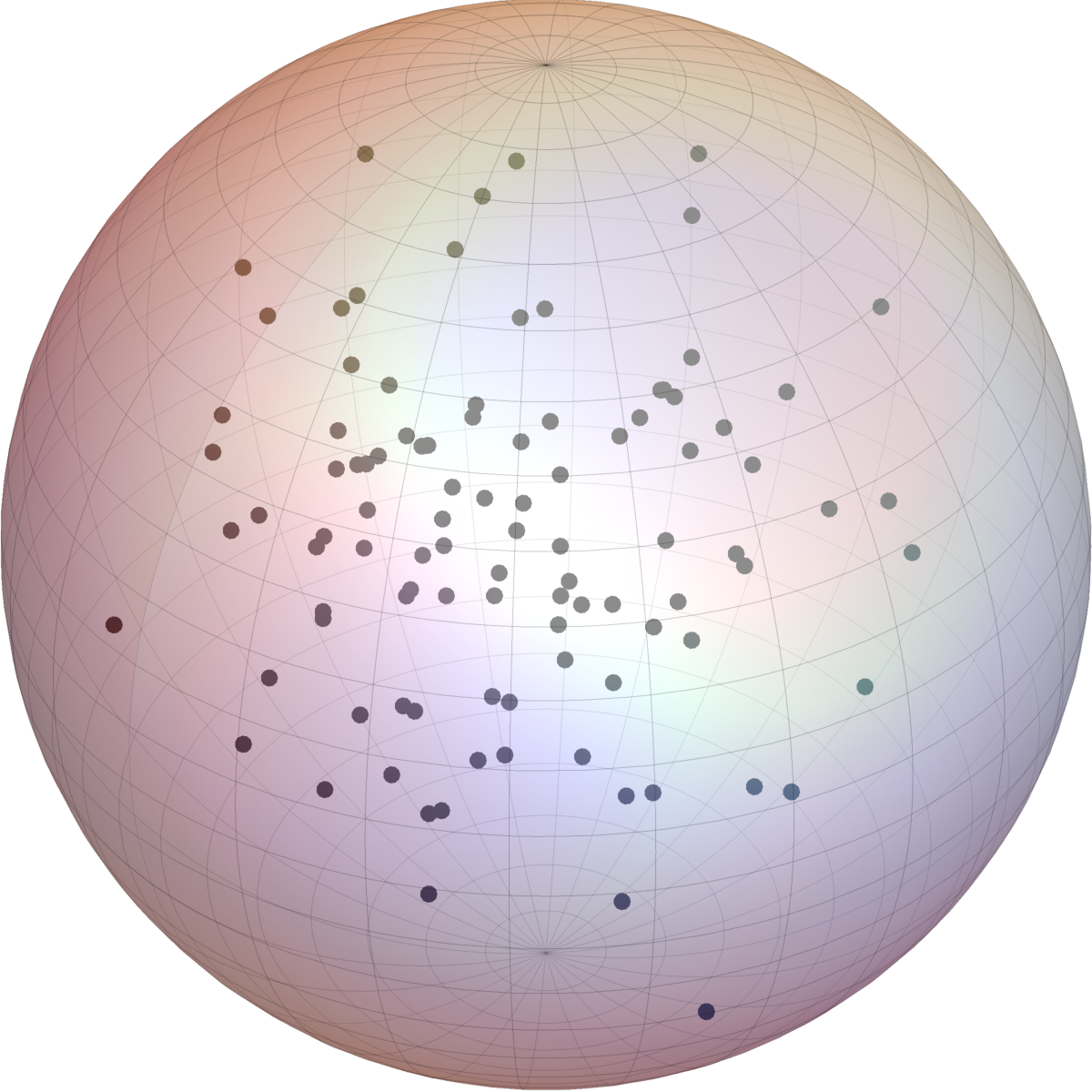}\\
    a)&b)\\
     \includegraphics[width=.3\textwidth]{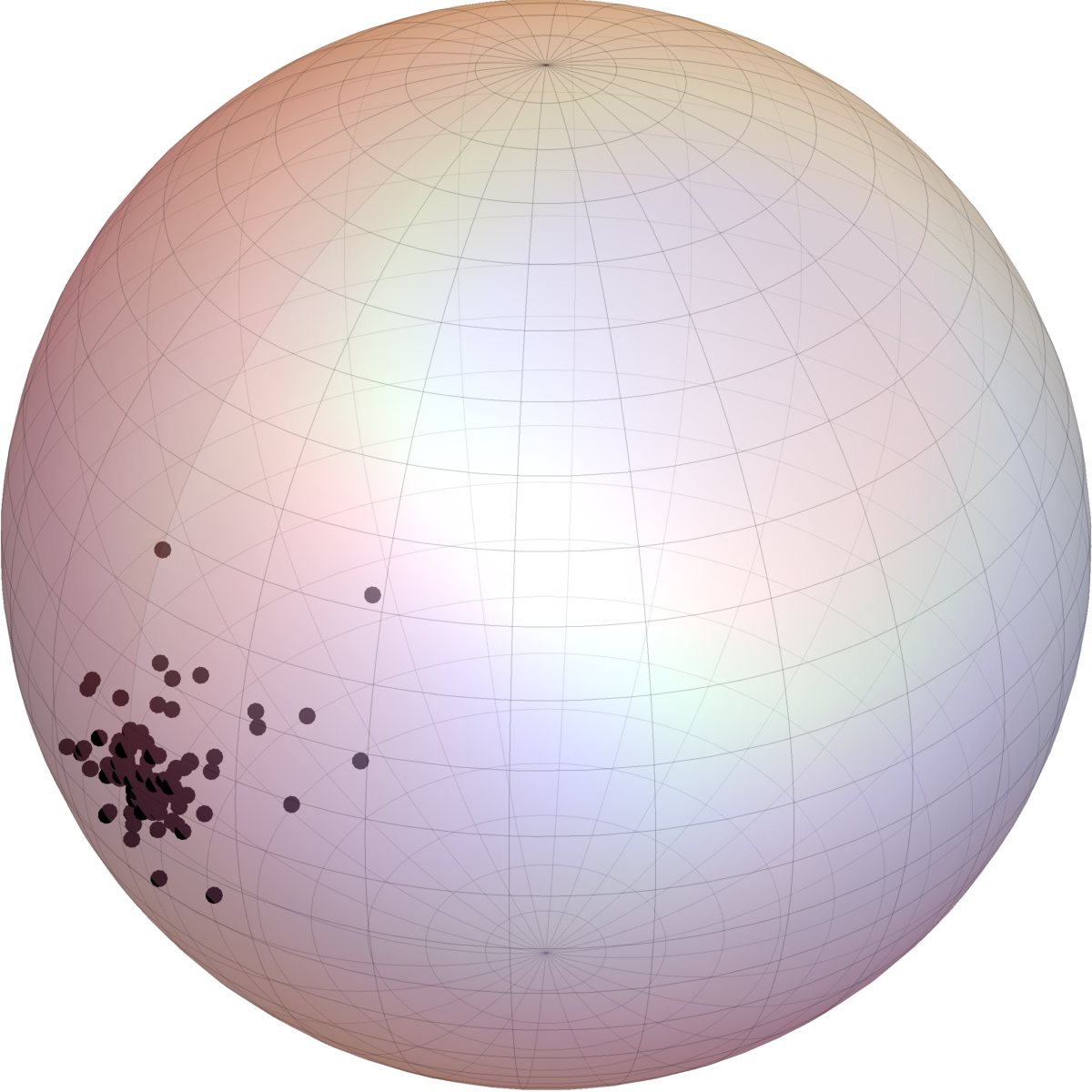}&\includegraphics[width=.3\textwidth]{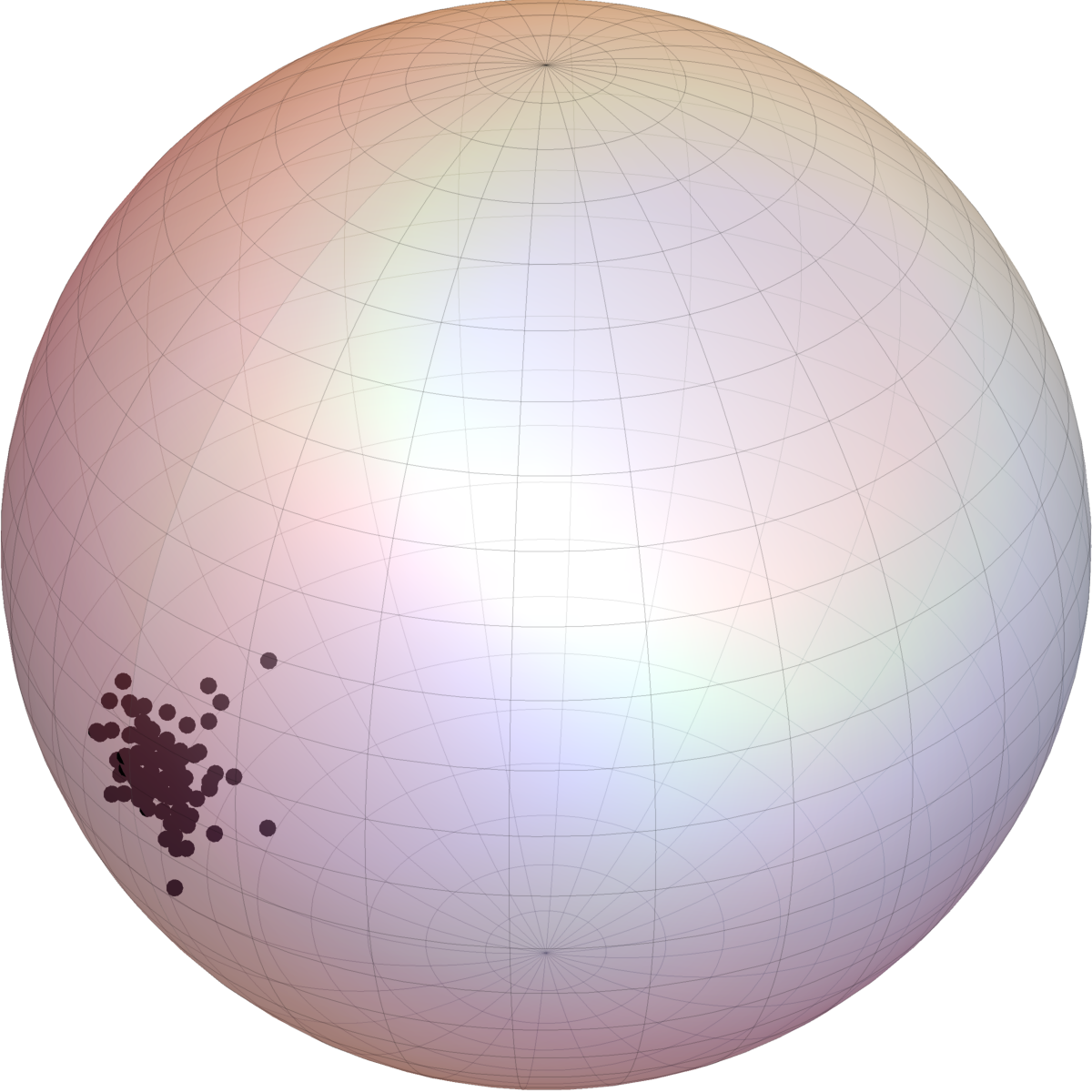}\\
    c)&d)\\
  \end{tabular}
  \caption{\label{fig4}Random samples from the Poincare ball $\mathbb{B}^3$ for parameter values: a) $a=(0,0,0)$, $s=3$; b) $a=(0,0,0)$, $s=5$; c) $a=(0.9,0,0)$, $s=3$; d) $a=(0.9,0,0)$, $s=5$.}
\end{figure*}

\subsection{Holomorphic barycenter in the Bergman ball}
\label{hol_bary_Berg}

The Bergman ball $\mathbb{D}^m$ in the complex vector space $\mathbb{C}^m$ is introduced in subsection \ref{Bergman_ball_subsect}. The group of holomorphic automorphisms of $\mathbb{D}^m$ is given by the formula \eqref{Bergman_transf}. We denote this group by $\mathbb{H}_m$. Notice that $\mathbb{D}^1$ is isomorphic to the Poincar\' e disc $\mathbb{B}^2$ with the symmetry group $\mathbb{H}_1 \sim \mathbb{G}_2$ consisting of transformations given by \eqref{Mobius}.

Holomorphic barycenter of the configuration $\{ \xi_1, \dots, \xi_N\}$ in $\mathbb{D}^m$ is defined as the minimum of the following function \cite{JacKal}
\begin{equation}
\label{potential_Bergman_ball}
M_m(a) = - \sum_{i=1}^N \log \frac{(1-|a|^2)(1-|\xi_i|^2)}{|1-\langle a,\xi_i \rangle|^2}.
\end{equation}

Holomorphic barycenter is holomorphicaly invariant, meaning that if $a$ is holomorphic barycenter of the configuration $\{ \xi_1,\dots,\xi_N \}$ and $m \in \mathbb{H}_m$, then $m(a)$ is the conformal barycenter of the configuration $\{m(\xi_1),\dots,m(\xi_N) \}$. We refer to \cite{JacKal} for the proof of this property.

We can construct Bergman swarms by computing infinitesimal generators of the group $\mathbb{H}_m$. The computation yields the following infinitesimal transformations:
$$
z \to W z, \mbox{ where } W \mbox{ is anti-Hermitian } m \times m \mbox{ matrix}
$$
and infinitesimal boost
$$
z \to 2(\langle z,w \rangle z - w), \mbox{ where } w \in \mathbb{D}^m
$$
and $\langle \cdot,\cdot \rangle$ denotes the standard scalar product in the complex vector space.

Hence, Bergman swarms are dynamical system in $\mathbb{D}^m$ of the following form
\begin{equation}
\label{infinitesimal_Berg_ball}
\dot z_j = K z_j + f - \langle z_j, f\rangle z_j, \quad j=1,\dots,N
\end{equation}
where $K$ is an anti-Hermitian $m \times m$ matrix and $f \equiv f(z_1,\dots,z_N)$ is a function taking values in $\mathbb{D}^m$.

We can assert the following:
\begin{lemma}
Consider dynamical system \eqref{infinitesimal_Berg_ball} with the initial conditions $z_1(0),\dots,z_N(0)$ in $\mathbb{D}^m$. One has that $z_j(t)=m_t(z_j(0))$ for a unique one-parametric family $m_t \in \mathbb{H}_m$.
\end{lemma}

As particular case, by setting
$$
K=0 \mbox {  and } w = \frac{K}{N} \sum_{k=1}^N z_i
$$
in \eqref{infinitesimal_Berg_ball} we obtain the following system
\begin{equation}
\label{Berg_swarm}
\dot z_j = \frac{K}{N} \sum_{k=1}^N z_k - \frac{K}{N} \left \langle z_j,\sum_{k=1}^N z_k \right \rangle z_j, \quad j=1,\dots,N.
\end{equation}

\begin{theorem}
The system \eqref{Berg_swarm} evolves by the actions of the group $\mathbb{H}_m$ of holomorphic automorphisms of $\mathbb{D}^m$. More precisely, there exists a one-parametric family $m_{a(t)} \in \mathbb{H}_m$, such that
$$
z_j(t) = m_t(z_j(0)), \mbox{ for } j=1,\dots,N, t>0.
$$
Furthermore, holomorphic barycenter $a(t)$ of configuration $\{ z_1(t),\dots,z_N(t)\}$ evolves by the following ODE:
\begin{equation}
\label{holom_bary_evol}
\frac{da}{dt} = \frac{K}{2N}(1-|a|^2) \sum_{i=1}^N m_a(z_i(0)),
\end{equation}
where $m_a$ is given by \eqref{Bergman_transf}.
\end{theorem}

Finally, the following theorem can be proven in an analogous way as its counterpart for the Poincar\' e disc.

\begin{theorem}
ODE \eqref{holom_bary_evol} is the gradient flow in metric \eqref{bergmet} for the potential \eqref{potential_Bergman_ball}.
\end{theorem}

The above theorems imply the method for computing the holomorphic barycenter of configuration $\{\xi_1,\dots,\xi_N\}$. Solve \eqref{Berg_swarm} with $K<0$ and initial conditions:
$$
z_1(0) = \xi_1,\dots,z_N(0)=\xi_N.
$$
Then for some $T$ $\{z_1(t) = m_t(\xi_1),\dots,z_N(t)=m_t(\xi_N)\}$ is a balanced configuration whenever $t>T$. The holomorphic barycenter of the configuration $\{\xi_1,\dots,\xi_N\}$ is found as the inverse of zero, that is $m^{-1}_t(0)$.

\begin{remark}
In this Section we have shown how particular Poincar\' e and Bergman swarms implement hyperbolic gradient descent algorithms for computation of the barycenters in the Poincar\' e and Bergman balls, respectively. More generally, swarms of the form \eqref{Poin_swarm_ball_gen} and \eqref{infinitesimal_Berg_ball} can be designed to implement other computations. For instance, functions $f$ in \eqref{Poin_swarm_ball_gen} and \eqref{infinitesimal_Berg_ball} can also be chosen to compute weighted barycenters.
\end{remark}

\section{Statistical models in hyperbolic balls}\label{sec:6}

In this Section we answer the questions $iii)$ and $iv)$ from Introduction for the case when $X$ is a hyperbolic ball. We extend the family of probability measures $Moeb_2(a,s)$ introduced in Section \ref{Moeb_Poin_disc} to Poincar\' e and Bergman balls.

\subsection{The M\" obius family of probability distributions in the Poincar\' e ball}

Let $\mathbb{B}^d$ be the Poincar\' e ball in $d$-dimensional real vector space. Consider the the following family of density functions
\begin{equation}
\label{conf_nat_Poin_ball}
p(x;a,s) = c_d (1-|g_a(x)|^2)^s, \quad x \in \mathbb{B}^d
\end{equation}
where $g_a$ is the M\" obius transformation \eqref{Mobius_ball} and $c_d$ is the normalizing constant.

Parameters of the family \eqref{conf_nat_Poin_ball} are $a \in \mathbb{B}^d$ and $s>d-1.$

The normalizing constant $c_d$ satisfies
$$
c_d \int_{\mathbb{B}^d} (1-|h_a(x)|^2)^s d \Lambda(x) = c_d \int_{\mathbb{B}^d} \left(\frac{(1-|x|^2)(1-|a|^2)}{\rho(a,x)}\right)^s (1-|x|^2)^{-d} dx  = 1,
$$
where $\rho(a,x)$ is defined by \eqref{rho}.

Evaluating the above integral yields
$$
c_d = \frac{\pi^{d/2} \Gamma(1+s-d/2)}{\Gamma(1+s-d)}.
$$
Putting all together, we rewrite \eqref{conf_nat_Poin_ball} to get the following family of densities:
\begin{equation}
\label{conf_nat_Poin_ball1}
p(x;a,s) = \frac{\pi^{d/2} \Gamma(1+s-d/2)}{\Gamma(1+s-d)} \left(\frac{(1-|x|^2)(1-|a|^2)}{\rho(a,x)}\right)^s, \quad x \in \mathbb{B}^d, \; s>d-1.
\end{equation}

We denote this family of probability distributions by $Moeb_d(a,s)$. and refer to them as {\it M\" obius distributions in the Poincar\' e ball}.

The family $Moeb_d(a,s^*)$ is conformally invariant for fixed $s^*>d-1$. This means that if $\mu \in Moeb_d(a,s^*)$ then $g_* \mu \in Moeb_d(a,s^*)$ for any $g \in \mathbb{G}_d$.

Moreover, for any $\mu_1,\mu_2 \in Moeb_d(a,s^*)$ there exists a M\" obius transformation $h \in \mathbb{G}_d$, such that $h_* \mu_1 = \mu_2$ and $h(a_1) = h(a_2).$

We further explain how to generate a random point $y \sim Moeb_d(a,s)$.

$i)$ First, consider the case $a=0$. Then the probability measure is rotationally symmetric (i.e. invariant w. r. to actions of the group of orthogonal transformations in $\mathbb{R}^d$). Hence, orientation of $y$ is uniformly distributed on the unit sphere $\mathbb{S}^{d-1}$.

In order to find the distribution for $|y|$, we denote by $\mathbb{B}_b^d \subseteq \mathbb{B}^d$ the ball with the radius $b<1$ and consider the following integral

\begin{eqnarray*}
\mathbb{P} \{ |y| < b \} &= &c_d \int_{\mathbb{B}_r^d} (1-|x|^2)^s (1-|x|^2)^{-d} dx = c_d \int_{\mathbb{B}_r^d} (1-|x|^2)^{s-d} dx\\
&=& c_d |\mathbb{S}^{d-1}| \int_0^b (1-r^2)^{s-d} r^{d-1} dr,
\end{eqnarray*}
where $|\mathbb{S}^{d-1}|$ is the area of the unit $d-1$-dimensional sphere.

Evaluating the above integral we finally obtain:
\begin{equation}
\label{prob_small_ball}
\mathbb{P} \{ |y| < b \} = \frac{2 \Gamma(1+s-d/2)}{\Gamma(1+s-d) \Gamma(d/2)} \; \frac{b^d}{d} \; _2 F_1(\frac{d}{2},d-s;\frac{d}{2}+1;b^2)
\end{equation}
where the notation $_2 F_1(\cdot,\cdot;\cdot;\cdot)$ stands for the Gaussian hypergeometric series.

Hence, in order to generate $y \sim Moeb_d(0,s)$, we first sample direction $u$ from the uniform distribution on $\mathbb{S}^{d-1}$ and the number $\kappa$ from the uniform distribution on $[0,1]$. We further solve the equation
\begin{equation}
\label{random_s_ball}
\frac{2 \Gamma(1+s-d/2)}{\Gamma(1+s-d) \Gamma(d/2)} \; \frac{b^d}{d} \; _2 F_1(\frac{d}{2},d-s;\frac{d}{2}+1;b^2) = \kappa
\end{equation}
with respect to $b$. Denote the solution of \eqref{random_s_ball} by $b^*$.

Finally, we set $y = b^* u$. This point is distributed as $Moeb(0,s)$.

$ii)$ Now, suppose that $y \sim Moeb_d(0,s)$. Then $h_a(y) \sim Moeb_d(a,s)$. Hence, in order to generate a random point from $Moeb_d(a,s)$ it suffices to sample a point $y \sim Moeb_d(0,s)$ and act on it by the M\" obius transformation $h_a$.

In Figure \ref{fig4} we plot randomly sampled points from the three-dimensional Poincar\' e ball for different values of $a$ and $s$.

Furthermore, let $\{x_1,\dots,x_N \}$ be observations in $\mathbb{B}^d$. Construct the maximum likelihood function
\begin{equation}
\label{ML_Poin_ball}
L(a,s|x_i) = \prod_{i=1}^N p(x_i;a,s) = \frac{\pi^{Nd/2} \Gamma(1+s-d/2)^N}{\Gamma(1+s-d)^N} \prod_{i=1}^N \left( \frac{(1-|x_i|^2)(1-|a|^2))}{\rho(a,x_i)} \right)^s.
\end{equation}
The log-likelihood function reads
\begin{eqnarray}
\label{log_likelihood_ball}
\frac{1}{N} \log L(a,s | x_i) =\frac{d}{2} \log \pi + \log \Gamma(1+s-d/2) - \log \Gamma(1+s-d) + \frac{s}{N} \sum_{i=1}^N \log \left( \frac{(1-|x_i|^2)(1-|a|^2))}{\rho(a,x_i)} \right).
\end{eqnarray}

Therefore, the maximum likelihood estimate of parameters $a$ and $s$ given the observations $x_1,\dots,x_N$ is the solution of the following maximization problem
\begin{eqnarray}
\label{max_log_likelihood}
\mbox{Maximize  } \log \Gamma(s+1-d/2) - \log \Gamma(s+1-d) + \frac{s}{N} \sum_{i=1}^N \log \frac{(1-|a|^2)(1-|x_i|^2)}{\rho(x_i,a)} = \nonumber \\
= \log \Gamma(s+1-d/2) - \log \Gamma(s+1-d) - \frac{s}{N} H_d(a), \mbox{ w. r. to  } a \in \mathbb{B}^d, \; s>d-1
\end{eqnarray}
where $H_d(a)$ is defined in \eqref{potential_Poin_ball}.

The optimization problem \eqref{max_log_likelihood} again splits into two separate problem for variables $a$ and $s$. The function $H_d(a)$ is given by \eqref{potential_Poin_ball} and its unique minimizer is the conformal barycenter of the configuration $\{ x_1,\dots,x_N \}$. The method of computation of this point is exposed in subsection \ref{conf_bary_ball}. Denote this point by $\hat a$.

Then the maximum likelihood estimation for $s$ is
\begin{equation}
\label{max_likelihood_s_ball}
\hat s = \textrm{argmax} \log \Gamma(s+1-\frac{d}{2}) - \log \Gamma(s+1-d) - \frac{s}{N} H_d(\hat a), \; s>d-1.
\end{equation}

It is easy to verify that the objective function in \eqref{max_likelihood_s_ball} is concave for $s>d-1$, hence, $\hat s$ is the unique solution of \eqref{max_likelihood_s_ball}.

\begin{remark}
Equations for sampling and maximum likelihood estimation procedures explained above are easily solvable. Since $d$ is an integer, the hypergeometric series in \eqref{random_s_ball} reduce to polynomials in $b$.

When it comes to the optimization problem \eqref{max_likelihood_s_ball}, logarithms of the Gamma function are well studied and convenient to differentiate. The Log Gamma function yields Stirling's series. Logarithmic derivatives of the Gamma function $\psi(x) = \frac{d}{dx} \log \Gamma(x) = \frac{\Gamma'(x)}{\Gamma(x)}$ arise in many applications and posses nice properties. These functions are efficiently computed using mathematical software packages.

For example, in the 3-dimensional ball, the maximum likelihood estimation $\hat s$ is the unique solution of the following equation
$$
\sum_{k=0}^\infty \frac{1}{s-2+k} - \sum_{k=0}^\infty \frac{1}{s-1/2+k} = \frac{1}{N} H(\hat a).
$$

For $d=4$, the equations for $\hat s$ becomes even simpler:
$$
\frac{1}{s-3} + \frac{1}{s-2} = \frac{1}{N} H(\hat a).
$$
\end{remark}

\subsection{Holomorphicaly natural family of probability distributions in the Bergman balls}

Introduce the family of probability distributions on the Bergman ball $\mathbb{D}^m \subset \mathbb{C}^m$ defined by densities
$$
p(z;a,s) = c_m \left( 1 - |f(z)|^2 \right)^s = c_m \left( \frac{(1-|a|^2)(1-|z|^2)}{|1-\langle a,z \rangle|^2} \right)^s, \quad s>m
$$
where $c_m$ is the normalizing constant.
Integration of the density function over the unit ball yields:
\begin{eqnarray*}
c_m \int_{{\mathbb B}_m} \left( 1 - |f(z)|^2 \right)^s d \Lambda(z)& =&c_m \int_{{\mathbb B}_m} \left( 1 - |f(z)|^2 \right)^s (1-|z|^2)^{-1-m} dz \\
&=&c_m \int_{{\mathbb B}_m} (1 - |z|^2)^s (1-|z|^2)^{-1-m} dz\\
&=&c_m \frac{2 \pi^m}{(m-1)!} \int_0^1 r^{2m-1} (1-r^2)^{s-m-1} dr.
\end{eqnarray*}
Equating the above integral to one, we find that
$$
c_m = \frac{\pi^{-m} \Gamma(s)}{\Gamma(s-m)}.
$$
In conclusion, we introduce the family of probability measures on balls $\mathbb{D}^m$ defined by densities of the form
\begin{equation}
\label{holom_nat_family}
p(z;a,s) = \frac{\pi^{-m} \Gamma(s)}{\Gamma(s-m)} \left( \frac{(1-|a|^2)(1-|z|^2)}{|1-\langle a,z \rangle|^2} \right)^s, \mbox{ where } s>m.
\end{equation}

We denote this family of probability distributions by $HolNat_m(a,s)$. We refer to them as {\it holomorphicaly natural distributions in the Bergman ball}. Sub-families $HolNat_m(a,s^*)$ are holomorphicaly invariant for fixed $s^*$. Moreover, the group $\mathbb{H}_m$ of holomorphic automorphisms acts transitively on these sub-families.

Let $\xi \sim HolNat_m(0,s)$ and denote by $\mathbb{D}^m_b$ the ball of radius $b<1$. Then the probability that $\xi \in \mathbb{D}^m_b$ reads
\begin{eqnarray*}
\mathbb{P}\{|\xi| < b \} &=& \frac{\pi^{-m} \Gamma(s)}{\Gamma(s-m)} \int_{\mathbb {D}^m_b} (1 - |z|^2)^s (1-|z|^2)^{-1-m} dz\\
&=&\frac{\pi^{-m} \Gamma(s)}{\Gamma(s-m)} \frac{2 \pi^m}{(m-1)!} \int_0^b r^{2m-1} (1-r^2)^{s-m-1} dr =\\
&=&\frac{\Gamma(s)}{(m-1)! \Gamma(s-m)} \; \frac{b^{2m}}{m} \; _2 F_1(m,m-s+1;m+1;b^2).
\end{eqnarray*}
Therefore, in order to sample from the probability distribution $HolNat_m(0,s)$, one needs to sample a random vector $v$ from the uniform distribution on the $m-1$-dimensional sphere $\mathbb{S}^{m-1} \subset \mathbb{C}^{m}$ and a real number $\kappa \sim U[0,1]$.

Further, denote by $b^*$ the (unique) solution of the equation
$$
\frac{\Gamma(s)}{(m-1)! \Gamma(s-m)} \; \frac{b^{2m}}{m} \; _2 F_1(m,m-s+1;m+1;b^2) = \kappa.
$$
Then $\xi = b^* v \sim HolNat_m(0,s)$.

Finally, in order to sample $\zeta \sim HolNat_m(a,s)$ with arbitrary $a \in \mathbb{D}^m$, we sample $\xi \sim HolNat_m(0,s)$ and let $\zeta = m_a(\xi)$, where $m_a$ is the holomorphic transformation of $\mathbb{D}^m$ defined by \eqref{Bergman_transf}.

Given the configuration $\{ \xi_1,\dots,\xi_N \}$ in $\mathbb{D}^m$ the maximum likelihood estimation of the parameter $a$ is the holomorphic barycenter of these points. The method of computation of this barycenter is explained in the subsection \ref{hol_bary_Berg}.

Once $\hat a$ is found, the maximum likelihood estimation for $s$ is the solution of the following maximization problem
\begin{equation}
\label{conc_MLE_Bergman}
\hat s = \mbox{ argmax } \log \Gamma(s) - \log \Gamma(s-m) - \frac{s}{N} M_m(\hat a) \mbox{ w.r. to } s > m.
\end{equation}
Again, the objective function is concave for $s>m$.

Differentiating the above function, we get the closed-form expression for $\hat s$:
$$
\sum_{k=1}^m \frac{1}{\hat s-k} = \frac{1}{N} M_m(\hat a).
$$

\begin{remark}
Although the two models (Poincar\' e and Bergman) of even-dimensional hyperbolic balls are not equivalent, our analysis shows that there is no qualitative difference between families $Moeb_d(a,s)$ and $HolNat_m(a,s)$ from the point of view of statistical modeling. Given the configuration of points in the even-dimensional ball, conformal and holomorphic barycenters are different points. However, it turns out that the choice of the model does not significantly affect neither computational efficiency, nor results.

By comparing optimization problems \eqref{max_likelihood_s_ball} and \eqref{conc_MLE_Bergman}, it is evident that estimations for the concentration parameter simply differ by $m-1$ for $d=2m$-dimensional balls in the real vector space.
\end{remark}

\begin{remark}
In conclusion, both families $Moeb_d(a,s)$ and $HolNat_m(a,s)$ allow for simple and efficient maximum likelihood estimation of parameters. This simplicity is consequence of the conformal invariance of the proposed statistical model. Notice, however, that statistical models $Moeb_d(a,s)$ and $HolNat_m(a,s)$ have very limited representative power. These families consist of unimodal and symmetric (in the hyperbolic metric) densities. One can only adjust the mean and concentration which are explicitly expressed in parameters of the models. Nevertheless, such a simple model is likely to be sufficient for many purposes, in particular, for encoding uncertainties in hyperbolic latent spaces in a number of setups. One can increase the representative power by using mixtures in order to approximate multimodal or skewed data.
\end{remark}

\section{Conclusion and outlook}\label{sec:7}

One remarkable trend in modern ML are explorations of the curvature and hidden symmetries of data sets, thus transcending beyond the traditional Euclidean paradigm. This motivated extensive research efforts on learning low-dimensional representations in Riemannian manifolds, thus giving a rise to the new paradigm which can be substantiated as {\it geometry informed ML}. ML in hyperbolic spaces is one of the most important and most challenging directions of research within that broad context. The underlying hypothesis is that structural information hidden in some ubiquitous data sets is best represented in negatively curved manifolds \cite{SGRS}. Although encouraging results have been reported in the previous decade, the real potential of hyperbolic ML is still to be examined in upcoming years. Further developments require systematization and enhancement of mathematical foundations. The present study is a contribution towards that goal, as we presented detailed and rigorous answers to four basic mathematical questions posed in Introduction. This mathematical framework may serve as a basis for principled group-theoretic approaches in hyperbolic ML.

One can build upon the concepts and methods presented here in order to design various algorithms in hyperbolic spaces, including expectation-maximization algorithm, Bayesian filtering, variational inference, normalizing flows, etc. Implementation of these algorithms would create a sufficient machinery for deep learning pipelines in hyperbolic spaces.

The minimal model of hyperbolic geometry is two-dimensional Riemannian manifold named the Poincar\' e disc. Although the main purpose of hyperbolic representations is reduction of the dimensionality by an order of magnitude, two dimensions may be insufficient in many setups. In the case that higher-dimensional hyperbolic manifolds are needed several options are available. Some experiments have been conducted with the hyperboloid model (Minkowski space) \cite{Poleksic,LLSZ,NK2}. Another option is to exploit multidisc, i.e. the product of several Poincar\' e discs \cite{TNN}. The present study focuses on hyperbolic balls, as we believe they provide the most convenient manifolds for the majority of hyperbolic ML tasks.

We analyzed two models of even-dimensional hyperbolic balls. Our goal was to identify which hyperbolic balls (Poincar\' e vs. Bergman) are more suitable for ML. The analysis presented in sections \ref{Bary_balls} and \ref{sec:6} demonstrates that choice of the metric does not make a significant difference from the computational and representational points of view. Therefore, any of the two models can be equally convenient for low-dimensional representations of hierarchical data.

In conclusion, hyperbolic ML opens truly exciting perspectives that are still to be explored. This is especially due to the fact that nowadays ML is broadly applied in many fields which are not very suitable for rigorous mathematical modeling. Hyperbolic representations have a potential to enable drastically more compact models of natural languages, networks, molecules, taxonomies. Furthermore, they can be advantageous when dealing with the data which are difficult to quantify, but exhibit a certain hierarchical structure. Examples of such data are opinions, political views or sentiments.

We hope that the present study will motivate further investigations and experiments in hyperbolic ML, based on the rigorous mathematical framework which encompasses Lie group theory and conformal geometry, along with optimization and statistical modeling.

%%% Uncomment this section and comment out the \bibliography{references} line above to use inline references.
% \begin{thebibliography}{1}

% 	\bibitem{kour2014real}
% 	George Kour and Raid Saabne.
% 	\newblock Real-time segmentation of on-line handwritten arabic script.
% 	\newblock In {\em Frontiers in Handwriting Recognition (ICFHR), 2014 14th
% 			International Conference on}, pages 417--422. IEEE, 2014.

% 	\bibitem{kour2014fast}
% 	George Kour and Raid Saabne.
% 	\newblock Fast classification of handwritten on-line arabic characters.
% 	\newblock In {\em Soft Computing and Pattern Recognition (SoCPaR), 2014 6th
% 			International Conference of}, pages 312--318. IEEE, 2014.

% 	\bibitem{hadash2018estimate}
% 	Guy Hadash, Einat Kermany, Boaz Carmeli, Ofer Lavi, George Kour, and Alon
% 	Jacovi.
% 	\newblock Estimate and replace: A novel approach to integrating deep neural
% 	networks with existing applications.
% 	\newblock {\em arXiv preprint arXiv:1804.09028}, 2018.

% \end{thebibliography}

\end{document}